%% file: autosam.tex
\documentclass[twocolumn]{autart}    

\usepackage{graphicx}          
\usepackage{amsmath,amssymb,amsfonts}
\usepackage[compress]{natbib}        
\usepackage{multirow}
\usepackage{booktabs}
\usepackage{mathtools}
\usepackage{tikz}
\usepackage{pgfplots}
\usepackage{url}
\usetikzlibrary{calc,shapes,arrows}
\usepackage{tikz-3dplot}
\usepackage{wrapfig}

\newcommand{\bbD}{\mathbb{D}}

\newcommand{\bbN}{\mathbb{N}}

\newcommand{\bbR}{\mathbb{R}}
\newcommand{\bbS}{\mathbb{S}}

\newcommand{\calC}{\mathcal{C}}
\newcommand{\calD}{\mathcal{D}}
\newcommand{\calE}{\mathcal{E}}
\newcommand{\calF}{\mathcal{F}}
\newcommand{\calG}{\mathcal{G}}

\newcommand{\calJ}{\mathcal{J}}

\newcommand{\calL}{\mathcal{L}}
\newcommand{\calM}{\mathcal{M}}

\newcommand{\calP}{\mathcal{P}}

\newcommand{\calR}{\mathcal{R}}

\newcommand{\bA}{\boldsymbol{A}}
\newcommand{\bB}{\boldsymbol{B}}
\newcommand{\bC}{\boldsymbol{C}}
\newcommand{\bD}{\boldsymbol{D}}

\newcommand{\bN}{\boldsymbol{N}}

\newcommand{\bP}{\boldsymbol{P}}

\newcommand{\bT}{\boldsymbol{T}}

\newcommand{\bX}{\boldsymbol{X}}

\newcommand{\bi}{\boldsymbol{i}}
\newcommand{\bj}{\boldsymbol{j}}

\newcommand{\br}{\boldsymbol{r}}
\newcommand{\bs}{\boldsymbol{s}}
\newcommand{\bt}{\boldsymbol{t}}

\newtheorem{theorem}{\bf Theorem}
\newtheorem{remark}[theorem]{\bf Remark}
\newtheorem{definition}[theorem]{\bf Definition}
\newtheorem{lemma}[theorem]{\bf Lemma}
\newtheorem{problem}[theorem]{\bf Problem}
\newtheorem{corollary}[theorem]{\bf Corollary}

\DeclareMathOperator{\NN}{NN}

\DeclareMathOperator{\diag}{diag}

\DeclareMathOperator{\blkdiag}{blkdiag}
\DeclareMathOperator{\chol}{chol}

\DeclareMathOperator{\Cayley}{Cayley}
\newcommand\signals[2]{\ell_{2e}^{#2}(\bbN_0^{#1})}
\newcommand{\lsb}{[}
\newcommand{\rsb}{]}

\newcommand{\sgrid}[1]{
      \coordinate (x) at #1;
      \draw[dotted] ($(x)+(-1.5,1)$) -- ($(x)+(1.5,1)$);
      \draw[dotted] ($(x)+(-1.5,-1)$) -- ($(x)+(1.5,-1)$);
      \draw[dotted] ($(x)+(-1,-1.5)$) -- ($(x)+(-1,1.5)$);
      \draw[dotted] ($(x)+(1,-1.5)$) -- ($(x)+(1,1.5)$);
}
\newcommand{\lgrid}[1]{
      \coordinate (x) at #1;
      \draw[dotted] ($(x)+(-2.3,1)$) -- ($(x)+(2.3,1)$);
      \draw[dotted] ($(x)+(-2.3,-1)$) -- ($(x)+(2.3,-1)$);
      \draw[dotted] ($(x)+(-2.3,2)$) -- ($(x)+(2.3,2)$);
      \draw[dotted] ($(x)+(-2.3,-2)$) -- ($(x)+(2.3,-2)$);
      \draw[dotted] ($(x)+(-1,-2.3)$) -- ($(x)+(-1,2.3)$);
      \draw[dotted] ($(x)+(1,-2.3)$) -- ($(x)+(1,2.3)$);
      \draw[dotted] ($(x)+(-2,-2.3)$) -- ($(x)+(-2,2.3)$);
      \draw[dotted] ($(x)+(2,-2.3)$) -- ($(x)+(2,2.3)$);
}

\definecolor{mycolor1}{rgb}{0.00000,0.44700,0.74100}%
\definecolor{mycolor2}{rgb}{0.85000,0.32500,0.09800}%
\definecolor{mycolor3}{rgb}{0.92900,0.69400,0.12500}%
\definecolor{mycolor4}{rgb}{0.49400,0.18400,0.55600}%
\definecolor{mycolor5}{rgb}{0.46600,0.67400,0.18800}%
\definecolor{mycolor6}{rgb}{0.30100,0.74500,0.93300}%
\definecolor{mycolor7}{rgb}{0.63500,0.07800,0.18400}%

\pgfplotsset{compat=1.18}

\begin{document}

\let\oldthebibliography=\thebibliography
\let\endoldthebibliography=\endthebibliography
\renewenvironment{thebibliography}[1]{%
    \begin{oldthebibliography}{#1}%
    \setlength{\parskip}{0ex}%
    \setlength{\itemsep}{0ex}%
}%
{%
    \end{oldthebibliography}%
}
\renewcommand*{\bibfont}{\normalfont\small}%

\begin{frontmatter}

\title{LipKernel: Lipschitz-Bounded Convolutional\\ Neural Networks via Dissipative Layers\thanksref{footnoteinfo}} 

\thanks[footnoteinfo]{P. Pauli was with the Institute for Systems Theory and Automatic Control, University of Stuttgart, while carrying out this work. F. Allgöwer acknowledges that this work was funded by Deutsche Forschungsgemeinschaft (DFG, German Research Foundation) under Germany's Excellence Strategy - EXC 2075 - 390740016 and under grant 468094890. P. Pauli thanks the International Max Planck Research School for Intelligent Systems (IMPRS-IS) for supporting her. The work of R. Wang and I. Manchester was supported in part by the Australian Research Council (DP230101014) and Google LLC.}

\author[Eindhoven]{Patricia Pauli},
\author[Sydney]{Ruigang Wang},
\author[Sydney]{Ian R. Manchester},
\author[Stuttgart]{Frank Allg\"ower}

\address[Eindhoven]{Department of Mechanical Engineering, Eindhoven University of Technology, 5600 MB Eindhoven, Netherlands,\\ (e-mail: p.d.pauli@tue.nl)}  
\address[Stuttgart]{Institute for Systems Theory and Automatic Control, University of Stuttgart, 70550 Stuttgart, Germany,\\ (e-mail: frank.allgower@ist.uni-stuttgart.de)}
\address[Sydney]{Australian Centre
for Robotics and School of Aerospace, Mechanical and Mechatronic Engineering, The University of Sydney, Australia, (e-mail: $\{$ruigang.wang, ian.manchester$\}$@sydney.edu.au)}


\begin{abstract}                          
We propose a novel layer-wise parameterization for convolutional neural networks (CNNs) that includes built-in robustness guarantees by enforcing a prescribed Lipschitz bound. Each layer in our parameterization is designed to satisfy a linear matrix inequality (LMI), which in turn implies dissipativity with respect to a specific supply rate. Collectively, these layer-wise LMIs ensure Lipschitz boundedness for the input-output mapping of the neural network, yielding a more expressive parameterization than through spectral bounds or orthogonal layers. Our new method LipKernel directly parameterizes dissipative convolution kernels using a 2-D Roesser-type state space model. This means that the convolutional layers are given in standard form after training and can be evaluated without computational overhead. In numerical experiments, we show that the run-time using our method is orders of magnitude faster than state-of-the-art Lipschitz-bounded networks that parameterize convolutions in the Fourier domain, making our approach particularly attractive for improving the robustness of learning-based real-time perception or control in robotics, autonomous vehicles, or automation systems. We focus on CNNs, and in contrast to previous works, our approach accommodates a wide variety of layers typically used in CNNs, including 1-D and 2-D convolutional layers, maximum and average pooling layers, as well as strided and dilated convolutions and zero padding. However, our approach naturally extends beyond CNNs as we can incorporate any layer that is incrementally dissipative.
\end{abstract}

\end{frontmatter}

\section{Introduction}
Deep learning architectures such as deep neural networks (NNs), convolutional neural networks (CNNs) and recurrent neural networks have ushered in a paradigm shift across numerous domains within engineering and computer science \citep{lecun2015deep}. Some prominent applications of such NNs include image and video processing tasks, natural language processing tasks, nonlinear system identification, and learning-based control \citep{bishop1994neural,li2021survey}. In these applications, NNs have been found to exceed other methods in terms of flexibility, accuracy, and scalability. However, as black box models, NNs in general lack robustness guarantees, limiting their utility for safety-critical applications.

In particular, it has been shown that NNs are highly sensitive to small ``adversarial'' input perturbations \citep{szegedy2013intriguing}. This sensitivity can be quantified by the Lipschitz constant of an NN. In learning-based control, ensuring safety and stability of closed-loop systems with a neural component often requires the gain of the NN to be bounded  \citep{berkenkamp2017safe,brunke2022safe,jin2020stability}, and the Lipschitz constant bounds the NN gain. Numerous approaches have been proposed for Lipschitz constant estimation \citep{virmaux2018lipschitz,combettes2020lipschitz,fazlyab2019efficient,latorre2020lipschitz}. While calculating the Lipschitz constant is an NP-hard problem \citep{virmaux2018lipschitz,jordan2020exactly}, computationally cheap but loose upper bounds are obtained as the product of the spectral norms of the matrices \citep{szegedy2013intriguing}, and much tighter bounds can be determined using semidefinite programming (SDP) methods derived from robust control \citep{fazlyab2019efficient, revay2020convex,latorre2020lipschitz,pauli2023lipschitz-ana,pauli2024lipschitz}.

While analysis of a given NN is of interest, it naturally raises the question of \textit{synthesis} of NNs with built-in Lipschitz bounds, which is the subject of the present work. 
Motivated by the composition property of Lipschitz bounds, most approaches assume 1-Lipschitz activation functions \citep{anil2019sorting, prach2022almost} and attempt to constrain the Lipschitz constant (i.e., spectral bound) of matrices and convolution operators appearing in the network. However, this can be conservative, resulting in limited expressivity, i.e., the constraints restrict the ability to fit the underlying function behavior.

To impose more sophisticated linear matrix inequality (LMI) based Lipschitz bounds, \citet{pauli2021training,pauli2022neural,gouk2021regularisation} include constraints or regularization terms into the training problem. However, the resulting constrained optimization problem tends to have a high computational overhead, e.g., due to costly projections or barrier calculations \citep{pauli2021training,pauli2022neural}. Alternatively, \citet{revay2020lipschitz,revay2023recurrent,wang2023direct,pauli2023lipschitz-syn} construct so-called \emph{direct} parameterizations that map free variables to the network parameters in such a way that LMIs are satisfied by design, which in turn ensures Lipschitz boundedness for equilibrium networks \citep{revay2020lipschitz}, recurrent equilibrium networks \citep{revay2023recurrent}, deep neural networks \citep{wang2023direct}, and 1-D convolutional neural networks \citep{pauli2023lipschitz-syn}, respectively. The major advantage of direct parameterization is that it poses the training of robust NNs as an unconstrained optimization problem, which can be tackled with existing gradient-based solvers. In this work, we develop a new direct parameterization for Lipschitz-bounded CNNs. 

Lipschitz-bounded convolutions can be parameterized in the Fourier domain, as in the Orthogonal and Sandwich layers in \citep{trockman2021orthogonalizing, wang2023direct}. However, this adds computational overhead of performing nonlinear operations or alternatively full-image size kernels leading to longer computation times for inference. In contrast, in this paper, we use a Roesser-type 2-D systems representation \citep{roesser1975discrete} for convolutions \citep{gramlich2023convolutional,pauli2024state}. This in turn allows us to directly parameterize the kernel entries of the convolutional layers, hence we denote our method as \textit{LipKernel}. This direct kernel parameterization has the advantage that at inference time we can evaluate convolutional layers of CNNs in standard form, which can be advantageous for system verification and validation processes. It also results in significantly reduced compute requirements for inference compared to Fourier representations, making our approach especially suitable for real-time control systems, e.g. in robotics, autonomous vehicles, or automation. Furthermore, LipKernel offers additional flexibility in the architecture choice, enabling pooling layers and any kind of zero-padding to be easily incorporated. 

Our work extends and generalizes the results in \citep{pauli2023lipschitz-syn} for parameterizing Lipschitz-bounded 1-D CNNs. In this work, we frame our method in a more general way than in \citep{pauli2023lipschitz-syn} such that \textit{any} dissipative layer can be included in the Lipschitz-bounded NN and we discuss a generalized Lipschitz property. In doing so, we include the concept of dissipativity into the synthesis problem, which we previously only discussed for analysis problems \citep{pauli2023lipschitz-ana}. We then focus the detailed derivations of our layer-wise parameterizations on the important class of CNNs. One main difference to \citep{pauli2023lipschitz-syn} and a key technical contribution of this work is the non-trivial extension from 1-D to 2-D CNNs, also considering a more general form including stride and dilation. Our parameterization relies on the Cayley transform, as also used in  \citep{trockman2021orthogonalizing,wang2023direct}. Additionally, we newly construct solutions for a specific 2-D Lyapunov equation for 2-D finite impulse response (FIR) filters, which we then leverage in our parameterization. 

The remainder of the paper is organized as follows. In Section \ref{sec:problem}, we state the problem and introduce feedforward NNs and all considered layer types. Section~\ref{sec:prelims} is concerned with the dissipation analysis problem used for Lipschitz constant estimation via semidefinite programming, followed by Section~\ref{sec:parameterization}, wherein we discuss our main results, namely the layer-wise parameterization of Lipschitz-bounded CNNs via dissipative layers. Finally, in Section~\ref{sec:experiments}, we demonstrate the advantage in run-time at inference time and compare our approach to other methods used to design Lipschitz-bounded CNNs.

\textbf{Notation:} By $I_n$, we mean the identity matrix of dimension $n$. We drop $n$ if the dimension is clear from context. By $\bbS^n$ ($\bbS_{++}^n$), we denote (positive definite) symmetric matrices and by $\bbD^n$ ($\bbD_{++}^n$) we mean (positive definite) diagonal matrices of dimension $n$, respectively. By $\chol(\cdot)$ we mean the Cholesky decomposition $L=\chol(A)$ of matrix $A = L^\top L$. 
Within our paper, we study CNNs processing image signals. For this purpose, we understand an image as a sequence $(u\lsb i_1,\ldots,i_d\rsb )$ with free variables $i_1,\ldots,i_d \in \bbN_0$. In this sequence, $u\lsb i_1,\ldots,i_d\rsb$ is an element of $\bbR^c$, where $c$ is called the channel dimension (e.g., $c = 3$ for RGB images). The \emph{signal dimension} $d$ will usually be $d=1$ for time signals (one time dimension) and $d=2$ for images (two spatial dimensions). The space of such signals/sequences is denoted by $\signals{d}{c} := \{ u: \bbN_0^d \to \bbR^c\}$. Images are sequences in $\signals{d}{c}$ with a finite square as support. For convenience, we sometimes use multi-index notation for signals, i.\,e., we denote $u\lsb i_1,\ldots,i_d\rsb$ as $u\lsb\bi\rsb$ for $\bi \in \bbN_0^d$. For these multi-indices, we use the notation $\bi + \bj$ for $(i_1+j_1,\ldots,i_d+j_d)$ and $\bi\bj = (i_1j_1,\ldots,i_dj_d)$. 
We further denote by $[\bi,\bj] = \{ \bt \in \bbN_0^d \mid \bi \leq \bt \leq \bj \}$ the \emph{interval} of all multi-indices between $\bi,\bj \in \bbN_0^d$ and by $|[\bi,\bj]|$ the number of elements in this set and the interval $[\bi,\bj[ = [\bi,\bj-1]$. By $\|\cdot\|$ we mean the $\ell_2$ norm of a signal, which reduces to the Euclidean norm for vectors, i.e., signals of length $1$, and $\Vert u \Vert_X^2 \coloneqq \sum_{\bi=0}^{\bN-1} u[\bi]^\top X u[\bi] $ is a signal norm weighted by some positive semidefinite matrix $X\succeq 0$ of signals of length $\bN$. 

\section{Problem Statement and Neural Networks}\label{sec:problem}
In this work, we consider deep NNs as a composition of $l$ layers
\begin{align}\label{eq:NN}
    \mathrm{NN}_\theta = \calL_l \circ \calL_{l-1} \circ \cdots \circ \calL_2 \circ \calL_1.
\end{align}
The individual layers $\calL_k$, $k=1,\dots,l$ encompass many different layer types, including but not limited to convolutional layers, fully connected layers, activation functions, and pooling layers. Some of these layers, e.g., fully connected and convolutional layers, are characterized by parameters $\theta_k, k=1,\dots,l,$ that are learned during training. In contrast, other layers such as activation functions and pooling layers do not contain tuning parameters.

The mapping from the input to the NN $u_1$ to its output $y_l$ is recursively given by
\begin{align}
    y_k &= \calL_k(u_k) & u_{k+1} &= y_k & k = 1,\ldots , l, \label{eq:inputOutputNN}
\end{align}
where $u_k \in \calD_{k-1}$ and $y_k \in \calD_k$ are the input and the output of the $k$-th layer and $\calD_{k-1}$ and $\calD_k$ the input and output domains. In \eqref{eq:inputOutputNN}, we assume that the output space of the $k$-th layer coincides with the input space of the $k+1$-th layer, which is ensured by reshaping operations at the transition between different layer types.

The goal of this work is to synthesize Lipschitz-bounded NNs, i.e., NNs of the form \eqref{eq:NN}, \eqref{eq:inputOutputNN} that satisfy the generalized Lipschitz condition
    \begin{equation}\label{eq:lipschitz}
        \Vert \NN_\theta(u^a)-\NN_\theta(u^b)\Vert_Q\leq \Vert u^a-u^b \Vert_R\quad \forall u^a,u^b\in\bbR^n.
    \end{equation}
for given $Q\in\bbS_{++}^{c_l}$, $R\in\bbS_{++}^{c_0}$ with input and output dimension $c_0$ and $c_l$ by design, and we call such NNs  $(Q,R)$-Lipschitz NNs. Choosing $Q=I$ and $R=\rho^2I$, we recover the standard Lipschitz inequality
    \begin{equation}\label{eq:lipschitz_2}
        \Vert \NN_\theta(u^a)-\NN_\theta(u^b)\Vert\leq \rho \Vert u^a-u^b \Vert \quad \forall u^a,u^b\in\bbR^n
    \end{equation}
with Lipschitz constant $\rho$. However, through our choice of $Q$ and $R$, we can incorporate domain knowledge and enforce tailored dissipativity-based robustness measures with respect to expected or worst-case input perturbations, including direction information. 
In this sense, we can view $\tilde{u}^\top \tilde{u}= u^\top X_0 u = u^\top L_0^\top L_0 u$, i.e, $\tilde{u}=L_0u$ as a rescaling of the expected input perturbation set to the unit ball. We can also weigh changes in the output of different classes (= entries of the output vector) differently according to their importance. \citet{singla2022improved} suggest a last layer normalization, which corresponds to rescaling every row of $W_l$ such that all rows have norm $1$, i.e., using $L_lW_l$ instead of $W_l$ with some diagonal scaling matrix $L_l$. We can interpret this scaling matrix $L_l^\top L_l$ as the output gain $X_l=L_l^\top L_l$.

\begin{rem}
   To parameterize input incrementally passive (i.e. strongly monotone) NNs, i.e., NNs with mapping $f:\calD_0\to\calD_l$ with equal input and output dimension $c_0=c_l$, which satisfy
    \begin{equation*}
        (u^a-u^b)^\top (f(u^a)-f(u^b)) \geq 0 \quad \forall u^a,u^b\in\bbR^n,
    \end{equation*}
    one can include a residual path $f(u) = \mu u + \NN_\theta(u)$ 
    with $\mu >0$ and constrain $\NN_\theta$ to have a Lipschitz bound $< \mu$, see e.g. \citep{chen2019residual,behrmann2019invertible,perugachi2021invertible,wang2024monotone}.
    Recurrent equilibrium networks extend this to dynamic models with more general incremental $(QSR)$-dissipativity \citep{revay2023recurrent}.
\end{rem}

\subsection{Problem statement}
To train a $(Q,R)$-Lipschitz NN, one can include constraints on the parameters $\theta=(\theta_k)_{k=1}^l$ in the underlying optimization problem to ensure the desired Lipschitz property. This yields a constrained optimization problem
\begin{equation}\label{eq:opt_con}
    \min_{\theta}\frac{1}{m}\sum_{i=1}^m \calJ(f(x^{(i)},\theta),y^{(i)}) \quad \text{s.\,t.}\quad \theta\in \Theta(Q,R),
\end{equation}
wherein $(x^{(i)},y^{(i)})_{i=1}^m$ are the training data, $\calJ(\cdot,\cdot)$ is the training objective, e.g., the mean squared error or the cross-entropy loss, and $\Theta(Q,R)$ is the set of parameters
\begin{equation*}
    \Theta(Q,R) \coloneqq \{\theta \mid \eqref{eq:lipschitz}~\text{for given}~ Q\in\bbS_{++}^{c_l},R\in\bbS_{++}^{c_0} \}.
\end{equation*}
It is, however, an NP-hard problem to find constraints $\theta\in \Theta(Q,R)$ that characterize all $(Q,R)$-Lipschitz NNs, and conventional characterizations by NNs with norm constraints are conservative. This motivates us to derive LMI constraints that hold for a large subset of $(Q,R)$-Lipschitz NNs.
\begin{problem}
    \label{problem-ana}
    Given some matrices $Q\in\bbS_{++}^{c_l}$ and $R\in\bbS_{++}^{c_0}$, identify a subset of the parameter set $\Theta(Q,R)$, described by LMIs, such that for all $\NN_\theta$ with weights satisfying these LMIs, the generalized Lipschitz inequality \eqref{eq:lipschitz} holds.
\end{problem}

To avoid projections or barrier functions to solve such a constrained optimization problem \eqref{eq:opt_con} as utilized in \citep{pauli2021training,pauli2022neural}, we instead use a direct parameterization $\phi\mapsto\theta$. This means that we parameterize $\theta$ in such a way that the underlying LMI constraints are satisfied by design. We can then train the Lipschitz-bounded NN by solving an unconstrained optimization problem
\begin{equation}\label{eq:opt}
    \min_{\phi}\frac{1}{m}\sum_{i=1}^m \calJ(f(x^{(i)},\theta(\phi)),y^{(i)}),
\end{equation}
using common first-order optimizers. In doing so, we optimize over the unconstrained variables $\phi\in\bbR^N$, while the parameterization ensures that the NN satisfies the LMI constraints, which in turn imply \eqref{eq:lipschitz}. This leads us to Problem \ref{problem1} of finding a direct parameterization $\phi\mapsto\theta$.

\begin{problem}
    \label{problem1}
    Given some matrices $Q\in\bbS_{++}^{c_l}$ and $R\in\bbS_{++}^{c_0}$, construct a parameterization $\phi\mapsto\theta$ for $\NN_\theta$ such that $\NN_\theta$ satisfies the generalized Lipschitz inequality \eqref{eq:lipschitz}.
\end{problem}

\subsection{CNN architecture}\label{sec:layer_types}
This subsection defines all relevant layer types for the parameterization of $(Q,R)$-Lipschitz CNNs. An example architecture of \eqref{eq:NN} is a classifying CNN
\begin{align*}
    \resizebox{0.47\textwidth}{!}{%
    $\mathrm{CNN}_\theta = \calF_{l} \circ \sigma \circ \cdots \circ \sigma \circ \calF_{p+1} \circ \calR  \circ \calP \circ \sigma \circ \calC_{p} \circ \cdots \circ \calP \circ \sigma \circ \calC_1$,}
\end{align*}
with $\calL_k\in\{\calF,\calC,\calP,\sigma,\calR\}$, wherein $\calF$ denote fully connected layers, $\calC$ denote convolutional layers, $\calP$ denote pooling layers, $\sigma$ denote activation functions, and $\calR$ denote reshape layers. In what follows, we formally define these layers.

\paragraph*{Convolutional layer}\label{sec:layer_def_conv}
A convolutional layer with layer index $k$
\begin{align*}
    \calC_k:\quad y_k = K_k \ast u_k + b_k,
\end{align*}
where $\ast$ denotes the convolution operator, is characterized by a convolution kernel $K_k$, and a bias term $b_k$, i.e., $\theta_k=(K_k,b_k)$. The input signal $u_k$ may be a 1-D signal, such as a time series, a 2-D signal, such as an image, or even a d-D signal.

For general dimensions $d$, a convolutional layer maps from $\calD_{k-1} = \signals{d}{c_{k-1}}$ to $\calD_k = \signals{d}{c_k}$. Using a compact multi-index notation, we write
\begin{align}
    y_{k}[\bi] = b_k + \sum_{\bt\in[0,\br_k]} K_k[\bt] u_k[\bs_k\bi - \bt], \label{eq:conv_dD}
\end{align}
where $u_k[\bs_k\bi - \bt]$ is set to zero if $\bs_k\bi - \bt$ is not in the domain of $u_k\lsb\cdot\rsb$ to account for possible zero-padding. The convolution kernel $K_k\lsb\bt\rsb\in\bbR^{c_{k} \times c_{k-1}}$ for $0\leq\bt\leq\br_k$ and the bias $b_k \in \bbR^{c_{k}}$ characterize the convolutional layer, and the stride $\bs_k$ determines by how many propagation steps the kernel is shifted along the respective propagation dimension.

\begin{remark}
    We use the causal representation \eqref{eq:conv_dD} for convolutional layers, i.e., $y_k[\bi]$ is evaluated based on past information. By shifting the kernel, we can retrieve an acausal representation 
    \begin{align}
        y_{k}[\bi] = b_k + \sum_{\bt\in[-\br_k/2, \br_k/2]} K_k[\bt] u_k[\bs_k\bi - \bt] \label{eq:conv_dD_shifted}
    \end{align}
    with symmetric kernels. The outputs of \eqref{eq:conv_dD} and \eqref{eq:conv_dD_shifted} are shifted accordingly.
\end{remark}

In this work, we focus on 1-D and 2-D CNNs whose main building blocks are 1-D and 2-D convolutional layers, respectively. A 1-D convolutional layer is a special case of \eqref{eq:conv_dD} with $d=1$, given by
\begin{align}
    y_{k}[i] = b_k + \sum_{t=0}^{r_k} K_k[t] u_k[s_ki - t]. \label{eq:conv_1D}
\end{align}

Furthermore, a 2-D convolutional layer ($d=2$) reads
\begin{equation}
    \begin{split}
        &y_{k}[i_1,i_2]  =  b_k + \\
        & \sum_{t_1\in[0,r_{k}^1]}\sum_{t_2\in[0,r_{k}^2]} K_k[t_1,t_2] u_k[s_{k}^1i_1 - t_1,s_{k}^2i_2-t_2],  \label{eq:conv_2D}
    \end{split}
\end{equation}
where $\br_k=(r_k^1,r_k^2)$ is the kernel size and $\bs_k=(s_k^1,s_k^2)$ the stride.

\paragraph*{Fully connected layer}
Fully connected layers  $\calF_k$ are static mappings with domain space $\calD_{k-1} = \bbR^{c_{k-1}}$ and image space $\calD_k = \bbR^{c_k}$ with possibly large channel dimensions $c_{k-1}, c_k$ (= neurons in the hidden layers). We define a fully connected layer as
\begin{align}\label{eq:fc}
    \calF_k : \bbR^{c_{k-1}} \to \bbR^{c_{k}}, u_k \mapsto y_k = b_k + W_k u_k
\end{align}
with bias $b_k \in \bbR^{c_{k}}$ and weight matrix $W_k \in \bbR^{c_{k} \times c_{k-1}}$, i.e., $\theta_k=(W_k,b_k)$.

\paragraph*{Activation function}
Convolutional and fully connected layers are affine layers that are typically followed by a nonlinear activation function. These activation functions $\sigma$ can be applied to both domain spaces $\calD_{k-1} = \bbR^{c_{k-1}}$ or $\calD_{k-1} = \signals{d}{c_{k-1}}$, but they necessitate $\calD_{k} \cong \calD_{k-1}$. Activation functions $\sigma : \bbR \to \bbR$ are applied element-wise to the input $u_k \in \calD_{k-1}$. For vector inputs $u_k \in \bbR^{c_k}$, $\sigma$ is then defined as
\begin{align*}
    \sigma: \bbR^{c_k} \to \bbR^{c_k}, u_{k} \mapsto y_{k} = 
    \begin{bmatrix}
        \sigma(u_{k1}) & \cdots & \sigma(u_{kc_k})
    \end{bmatrix}^\top.
\end{align*}
Furthermore, we lift the scalar activation function to signal spaces $\signals{d_{k-1}}{c_{k-1}}$, which results in $\sigma : \signals{d}{c_k} \to \signals{d}{c_k}$,
\begin{align*}
    (u_k[\bi])_{\bi \in \bbN_0^{d}} \mapsto (y_k[\bi])_{\bi \in \bbN_0^{d}} = (\sigma(u_k[\bi]))_{\bi \in \bbN_0^{d}}.
\end{align*}

\paragraph*{Pooling layer}
A convolutional layer may be followed by an additional pooling layer $\calP$, i.\,e., a downsampling operation from $\calD_{k-1} = \signals{d}{c_{k}}$ to $\calD_k = \signals{d}{c_k}$ that is applied channel-wise. Pooling layers generate a single output signal entry $y[\bi]$ from the input signal batch $(u_k\lsb\bs_k\bi - \bt\rsb \mid \bt \in [0,\br_k])$. The two most common pooling layers are average pooling $\calP^{\mathrm{av}} : \signals{d}{c_k} \to \signals{d}{c_k}$,
\begin{align*}
    y_k\lsb\bi\rsb \coloneqq& \mathrm{mean} (u_k\lsb\bs_k\bi - \bt\rsb \mid \bt \in [0,\br_k])\\
    =& \frac{1}{|[0,\br_k]|} \sum_{0 \leq \bt \leq \br_k} u_k\lsb\bs_k\bi - \bt\rsb
\end{align*}
and maximum pooling $\calP^{\mathrm{max}} : \signals{d}{c_k} \to \signals{d}{c_k}$,
\begin{align*}
    y_k\lsb\bi\rsb &\coloneqq \max (u_k\lsb\bs_k\bi - \bt\rsb \mid \bt \in [0,\br_k]),
\end{align*}
where the maximum is applied channel-wise. Other than \citep{pauli2023lipschitz-syn}, we allow for all $\bs_k\leq\br_k$, meaning that the kernel size is either larger than the shift or the same.

\paragraph*{Reshape operation}
An NN \eqref{eq:NN} may include signal processing layers such as convolutional layers and layers that operate on vector spaces, such as fully connected layers. At the transition of such different layer types, we require a reshape operation
\begin{align*}
    \calR & : \signals{d}{c} \to \bbR^{|\lsb 0,\bN\rsb| \cdot c},~(y[\bi] )_{\bi\in\bbN_0^{d}} \mapsto \calR(y),
\end{align*}
that flattens a signal into a vector 
\begin{equation*}
\calR(y_{k}) =
    \begin{bmatrix}
        y_{k}[0]^\top & \dots & y_{k}[\bN_k]^\top
    \end{bmatrix}^\top = u_{k+1}
\end{equation*}
or vice versa, a vector into a signal. 

\section{Dissipation Analysis of Neural Networks}\label{sec:prelims}
Prior to presenting the direct parameterization of Lipschitz-bounded CNNs in Section \ref{sec:parameterization}, we address Problem \ref{problem-ana} of characterizing $(Q,R)$-Lipschitz NNs by LMIs in this section. In Subsection~\ref{sec:dissipative_layers}, we first discuss incrementally dissipative layers, followed by Subsection~\ref{sec:statespace}, wherein we introduce state space representations of the Roesser type for convolutions. In Subsection~\ref{sec:QCs_slope_restriction}, we then state quadratic constraints for slope-restricted nonlinearities and discuss layer-wise LMIs that certify dissipativity for the layers and \eqref{eq:lipschitz} for the CNN in Subsection~\ref{sec:layerwise_LMIs}. 
Throughout this section, where possible, we drop layer indices for improved readability. The subscript ``$-$'' refers to the previous layer; for example, $c$ is short for $c_k$, and $c_-$ is short for $c_{k-1}$.

\subsection{Incrementally dissipative layers}\label{sec:dissipative_layers}
To design Lipschitz-bounded NNs, previous works have parameterized the individual layers of a CNN to be orthogonal or to have constrained spectral norms \citep{anil2019sorting,trockman2021orthogonalizing,prach2022almost}, thereby ensuring that they are 1-Lipschitz. An upper bound on the Lipschitz constant of the end-to-end mapping is then given by
\begin{equation*}
    \rho = \prod_{k=1}^l \mathrm{Lip}(\calL_k),
\end{equation*}
where $\mathrm{Lip}(\calL_k)$ are upper Lipschitz bounds for the $k=1,\dots,l$ layers. In contrast, our approach does not constrain the individual layers to be orthogonal but instead requires them to be incrementally dissipative \citep{byrnes1994losslessness}, thus providing more degrees of freedom while also guaranteeing a Lipschitz upper bound for the end-to-end mapping.

\begin{definition}[Incremental dissipativity]\label{def:dissipativity}
    A layer $\calL_k: \calD_{k-1}\to\calD_{k}: u_k\mapsto y_k$ is incrementally dissipative with respect to a supply rate $s(\Delta u_k[\bi], \Delta y_k[\bi])$ if for all inputs $u_k^a,u_k^b\in\calD_{k-1}$ and all $\bN_k\in\bbN_0^d$
    \begin{equation} \label{eq:dissipativity}
        \sum_{\bi\in[0,\bN_k-1]} s(\Delta u_k[\bi], \Delta y_k[\bi]) \geq 0,
    \end{equation}
    where $\Delta u_k[\bi] = u_k^a[\bi]-u_k^b[\bi]$, $\Delta y_k[\bi] = y_k^a[\bi]-y_k^b[\bi]$.
\end{definition}
In particular, we design layers to be incrementally dissipative with respect to the supply
\begin{equation}\label{eq:supply}
    \sum_{\bi=0}^{\bN_k-1} s(\Delta u_k[\bi], \Delta y_k[\bi])=\Vert\Delta u_k \Vert_{X_{k-1}}^2 - \Vert \Delta y_k \Vert_{X_k}^2,
\end{equation}
which can be viewed as a generalized incremental gain/Lipschitz property with directional gain matrices $X_k\in\bbS_{++}^{c}$ and $X_{k-1}\in\bbS_{++}^{c_{k-1}}$. Note that \eqref{eq:dissipativity} includes vector inputs, in which case $\bN_k=0$. Furthermore, our approach naturally extends beyond the main layer types of CNNs presented in Section \ref{sec:layer_types} as \emph{any} function that is incrementally dissipative with respect to \eqref{eq:supply} can be included as a layer $\calL_k$ into a $(Q,R)$-Lipschitz feedforward NN \eqref{eq:NN}.

\begin{figure}[h]
\centering
    \resizebox{0.48\textwidth}{!}{%
    \begin{tikzpicture}
      \path[thick,->] (-2,0) edge (2,0);
      \path[thick,->] (0,-2) edge (0,2);
      \sgrid{(0,0)};
      \draw[thick] (0,0) circle (1);
      \fill[mycolor1] (0,0) circle (1);
      \draw[thick] (1,-0.2) edge (1,0.2);
      \draw (1.2,0.3) node {1};
      \draw[thick] (-0.2,1) edge (0.2,1);
      \draw (0.3,-3.8) node {1};
      \draw (2,-0.3) node {$u_1^{(1)}$};  
      \draw (0.5,1.9) node {$u_1^{(2)}$};

      \path[thick,->, bend left=22.5] (1.8,1) edge (3.8,1);
      \draw (2.8,1.5) node {$\calL_1$};

      \path[thick,->] (4,0) edge (8.8,0);
      \path[thick,->] (6.4,-2.4) edge (6.4,2.4);
      \lgrid{(6.4,0)};
      \draw (8.3,-0.3) node {$y_1^{(1)}=u_2^{(1)}$};   
      \draw (7.4,2.3) node {$y_1^{(2)}=u_2^{(2)}$}; 
      \draw[thick,rotate around={-45:(6.4,0)}] (6.4,0) ellipse (2 and 0.5);
      \fill[mycolor1, rotate around={-45:(6.4,0)}] (6.4,0) ellipse (0.9 and 0.4);
      (6.4,0) edge (8.4,0);
      (6.4,0) edge (8.31,0);
      edge (6.4,0.5);

      \path[thick,->, bend left=22.5] (9,1) edge (11,1);
      \draw (10,1.5) node {$\calL_2$};

      \path[thick,->] (11,0) edge (15,0);
      \path[thick,->] (13,-2) edge (13,2);
      \sgrid{(13,0)};
      \draw[thick] (13,0) circle (1);
      \fill[mycolor1] (13.1,-0.1) circle (0.55);
      \fill[mycolor1] (13.1,0.15) circle (0.55);
      \fill[mycolor1] (12.9,-0.15) circle (0.55);
      \fill[mycolor1] (12.9,0.1) circle (0.55);
      \draw (15,-0.3) node {$y_2^{(1)}$};  
      \draw (13.5,1.9) node {$y_2^{(2)}$};      
      \path[thick,->] (-2,-5) edge (2,-5);
      \path[thick,->] (0,-7) edge (0,-3);
      \sgrid{(0,-5)};
      \draw[thick] (0,-5) circle (1);
      \fill[mycolor1] (0,-5) circle (1);
      \draw[thick] (1,-5.2) edge (1,-4.8);
      \draw (1.2,-4.7) node {1};
      \draw[thick] (-0.2,-4) edge (0.2,-4);
      \draw (0.3,1.2) node {1};
      \draw (2,-5.3) node {$u_1^{(1)}$};  
      \draw (0.5,-3.1) node {$u_1^{(2)}$};

      \path[thick,->, bend left=22.5] (1.8,-4) edge (3.8,-4);
      \draw (2.8,-3.5) node {$\calL_1$};

      \path[thick,->] (4.4,-5) edge (8.4,-5);
      \path[thick,->] (6.4,-7) edge (6.4,-3);
      \sgrid{(6.4,-5)};
      \draw (8.3,-5.3) node {$y_1^{(1)}=u_2^{(1)}$};   
      \draw (7.4,-3.2) node {$y_1^{(2)}=u_2^{(2)}$};   
      \draw[thick] (6.4,-5) circle (0.94);
      \fill[mycolor1, rotate around={-45:(6.4,-5)}] (6.4,-5) ellipse (0.9 and 0.4);

      \path[thick,->, bend left=22.5] (8.5,-4) edge (10.5,-4);
      \draw (9.5,-3.5) node {$\calL_2$};

      \path[thick,->] (10.8,-5) edge (15.2,-5);
      \path[thick,->] (13,-7.2) edge (13,-2.8);
      \lgrid{(13,-5)};
      \draw[thick] (13,-5) circle (1.882);
      \fill[mycolor1] (13.1,-5.1) circle (0.55);
      \fill[mycolor1] (13.1,-4.85) circle (0.55);
      \fill[mycolor1] (12.9,-5.15) circle (0.55);
      \fill[mycolor1] (12.9,-4.9) circle (0.55);
      \draw (15.4,-5.3) node {$y_2^{(1)}$};  
      \draw (13.5,-2.7) node {$y_2^{(2)}$};      
    \end{tikzpicture}}
    \caption{For $\calF_2\circ\sigma\circ\calF_1$ with $c_0=c_1=c_2=2$, we compare over-approximations for reachability sets shown in blue, we obtain ellipsoidal sets using incrementally dissipative layers (top) and circles using Lipschitz bounds (bottom).}
    \label{fig:directional_gains}
\end{figure}
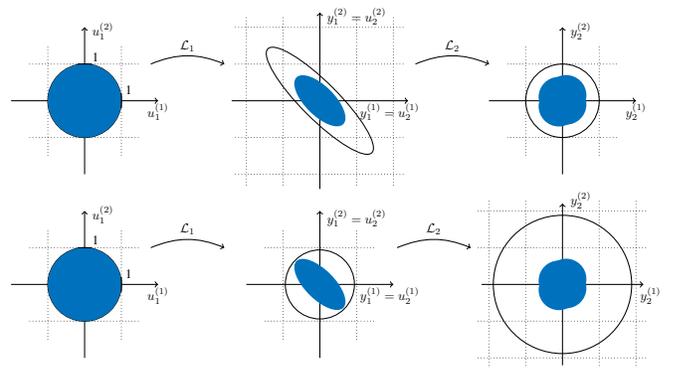

Fig.~\ref{fig:directional_gains} illustrates the additional degrees of freedom gained by considering incrementally dissipative layers rather than Lipschitz-bounded layers considering a fully connected, two-layer NN with an input, hidden and output dimension of two. For input increments taken from a unit ball, we find an ellipse $\calE=\{y_1^a,y_1^b\in\bbR^{c} \mid (y_1^a-y_1^b)^\top X_1 (y_1^a-y_1^b)\leq 1 \}$ that over-approximates the reachability set in the hidden layer or a Lipschitz bound that characterizes a circle for this purpose, respectively. 
The third and final reachability set is a circle scaled by a Lipschitz upper bound. This set is created using either the ellipse characterized by $X_1$ or the circle of the previous layer as inputs. These input sets were chosen such that the final reachability set is minimized. We see a clear difference between the two approaches. Using ellipses obtained by the incremental dissipativity approach, the Lipschitz bound is $1$, using circles as in the Lipschitz approach, it is almost 2, illustrating that we can find tighter Lipschitz upper bounds.

\begin{figure}
    \input{figs/motivational_example.tex}
    \caption{Fit of a cosine function using NN from LMI-based parameterization with dissipative layers and an NN with 1-Lipschitz layers with weights which are constrained to have spectral norm $1$.}
    \label{fig:cosine_motivaitional}
\end{figure}
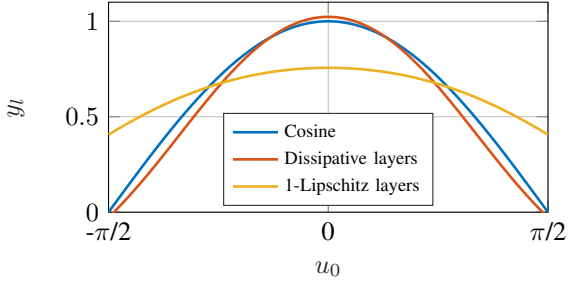

In NN design this translates into a higher model expressivity. To illustrate this, we consider the regression problem of fitting the cosine function between $\frac{-\pi}{2}$ and $\frac{\pi}{2}$, also utilized in \citep{pauli2021training}. We use a simple NN architecture $\calF_2\circ\sigma\circ\calF_1$ with $c_0=c_2=1$, $c_1=2$, and activation function $\mathrm{tanh}$ and construct weights and biases as
\begin{equation*}
 W_1 = \begin{bmatrix}
     -1\\
     -1
 \end{bmatrix},~
 b_1 = \begin{bmatrix}
     -1\\
     1
 \end{bmatrix},~
 W_2 = \begin{bmatrix}
    -1 & 1
 \end{bmatrix},~
 b_2 = -0.5.
\end{equation*}
Both layers are incrementally dissipative and the weights satisfy LMI constraints that verify that the end-to-end mapping is guaranteed to be $1$-Lipschitz, as will be discussed in detail in Subsection \ref{sec:layerwise_LMIs}. Clearly the weights have spectral norms of $\sqrt{2}$, meaning that the individual layers are \emph{not} $1$-Lipschitz. Next, we construct an NN to best fit the cosine with $1$-Lipschitz weights obtained by spectral normalization as
\begin{align*}
 &W_1 = \frac{1}{\sqrt{2}}\begin{bmatrix}
     -1\\
     -1
 \end{bmatrix},~
 b_1 = \begin{bmatrix}
     -\sqrt{2}\\
     \sqrt{2}
 \end{bmatrix},\\
 &W_2 = \frac{1}{\sqrt{2}}\begin{bmatrix}
    -1 & 1
 \end{bmatrix},~
 b_2 = -0.5.
\end{align*}
In Fig.~\ref{fig:cosine_motivaitional}, we see the resulting fit of the two functions and a clear advantage in expressiveness of the LMI-based parameterization using dissipative layers.

\subsection{State space representation for convolutions}\label{sec:statespace}

In kernel representation \eqref{eq:conv_dD}, convolutions are not amenable to SDP-based methods. However, they can be reformulated as fully connected layers via Toeplitz matrices \citep{Goodfellow-et-al-2016,pauli2022neural,aquino2022robustness}, parameterized in the Fourier domain \citep{wang2023direct}, or represented in state space \citep{gramlich2023convolutional,pauli2024state}. In what follows, we present compact state space representations for 1-D and 2-D convolutions derived in \citep{pauli2024state}, which allow the direct parameterization of the kernel parameters.

\paragraph*{1-D convolutions}
A possible discrete-time state space representation of a convolutional layer \eqref{eq:conv_1D} with stride $s=1$ is given by
\begin{equation}\label{eq:ss_conv}
    \begin{split}
        x[i+1] &= \bA x[i] + \bB u[i],\\ 
        v[i]   &= \bC x[i] + \bD u[i] + b,
    \end{split}
\end{equation}
with initial condition $x[0]=0$, where
\begin{equation}\label{eq:ABCD}
\begin{aligned}
    &\bA =
    \begin{bmatrix}
        0 & I_{c_-(r-1)} \\
        0 &   0 \\
    \end{bmatrix},
    &\bB &= \begin{bmatrix}
        0 \\
        I_{c_-}
    \end{bmatrix},\\
    &\bC =
    \begin{bmatrix}
        K[r] & \cdots & K[1]
    \end{bmatrix},
    &\bD &= K[0].
\end{aligned}
\end{equation}
In \eqref{eq:ss_conv}, we denote the state, input, and output by $x[i]\in\bbR^{n}$, $u[i]\in\bbR^{c_-}$ and $y[i]\in\bbR^{c}$, respectively, and the state dimension is $n=rc_{-}$.

It should be noted that in the state space representation \eqref{eq:ABCD}, all parameters $K[i]$, $i=0,\dots,r$ are collected in the matrices $\bC$ and $\bD$, and $\bA$ and $\bB$ account for shifting previous inputs into memory. Accordingly, $\bC$ and $\bD$ are parameterized in Section~\ref{sec:parameterization_conv}.

\paragraph*{2-D convolutions}\label{sec:statespace_2D}
We describe a 2-D convolution using the Roesser model \citep{roesser1975discrete}
\begin{equation}
	\begin{aligned}
		\begin{bmatrix}
			x_1[i+1,j]\\
			x_2[i,j+1]
		\end{bmatrix}
		&=
            \underbrace{\begin{bmatrix}
                A_{11} & A_{12} \\
                A_{21} & A_{22}
            \end{bmatrix}}_{\eqqcolon\bA}
		\begin{bmatrix}
			x_1[i,j]\\
			x_2[i,j]
		\end{bmatrix}
		+
		\underbrace{\begin{bmatrix}
                B_1\\
                B_2
            \end{bmatrix}}_{\eqqcolon\bB}
            u[i,j]\\
		y[i,j]&=
            \underbrace{\begin{bmatrix}
                C_1 & C_2
            \end{bmatrix}}_{\eqqcolon\bC}
		\begin{bmatrix}
			x_1[i,j]\\
			x_2[i,j]
		\end{bmatrix}
		+ \underbrace{D}_{\eqqcolon\bD} u[i,j]+b.
	\end{aligned}
	\label{eq:RoesserSys}
\end{equation}
with states $x_1[i,j]\in\bbR^{n_{1}}$, $x_2[i,j]\in\bbR^{n_{2}}$, input $u[i,j]\in\bbR^{n_{u}}$, and output $y[i,j]\in\bbR^{n_{y}}$. A possible state space representation for the 2-D convolution \eqref{eq:conv_2D} is given by Lemma~\ref{lem:min_real_2D} \citep{pauli2024state}.

\begin{lemma}[Realization of 2-D convolutions]\label{lem:min_real_2D}
    Consider a convolutional layer $\calC : \signals{2}{c_{-}} \to \signals{2}{c}$ with representation \eqref{eq:conv_2D} and stride $s_1=s_2=1$ characterized by the convolution kernel $K$ and the bias $b$. This layer is realized in state space \eqref{eq:RoesserSys} by the matrices
    \begin{equation}\label{eq:ABCD_2D}
    \begin{split}
        &\left[
        \begin{array}{c|c}
            A_{12} & B_1 \\ \hline
            C_2 & D
        \end{array}
        \right]
        \!=\!\left[
        \begin{array}{ccc|c}
            K\lsb r_1,r_2 \rsb & \cdots & K \lsb r_1,1 \rsb & K \lsb r_1,0 \rsb \\
            \vdots & \ddots & \vdots & \vdots \\
            K\lsb 1,r_2 \rsb & \cdots & K \lsb 1,1 \rsb & K \lsb 1,0 \rsb \\ \hline
            K\lsb 0,r_2 \rsb & \cdots & K \lsb 0,1 \rsb & K \lsb 0,0 \rsb \\
        \end{array}
        \right]\!,\\
        &\left[
        \begin{array}{c}
            A_{11} \\ \hline
            C_1
        \end{array}
        \right]
        =
        \left[
        \begin{array}{cc}
            0 & 0\\
            I_{c(r_1-1)} & 0\\ \hline
            0 & I_{c}
        \end{array}
        \right],~A_{21} = 0,\\
        &\left[
        \begin{array}{c|c}
            A_{22} & B_2
        \end{array}
        \right]
        =
        \left[
        \begin{array}{cc|c}
            0 & I_{c_-(r_2-1)} & 0\\
            0 & 0 & I_{c_-}
        \end{array}
        \right],
    \end{split}
    \end{equation}
    where $K[i_1,i_2]\in\bbR^{c\times c_{-}},~i_1 = 0,\dots,r_1,~i_2= 0,\dots,r_2$ with initial conditions $x_1\lsb 0,i_2\rsb = 0$ for all $i_2\in\bbN_0$, and $x_2\lsb i_1,0\rsb = 0$ for all $i_1\in\bbN_0$. The state, input, and output dimensions are $n_{1}=c r_1$, $n_{2}=c_{-}r_2$, $n_u=c_{-}$, $n_y=c$.
\end{lemma}

\begin{rem}\label{rem:ss_stride}
For stride $\bs>1$, \citep{pauli2024state} constructs state space representations \eqref{eq:ABCD} and \eqref{eq:ABCD_2D}, based on which our parameterization directly extends to strided convolutions.
\end{rem}

\subsection{Slope-restricted activation functions}\label{sec:QCs_slope_restriction}
The nonlinear and large-scale nature of NNs often complicates their analysis. However, over-approximating activation functions with quadratic constraints enables SDP-based Lipschitz estimation and verification \citep{fazlyab2020safety,fazlyab2019efficient}. Common activations like $\mathrm{ReLU}$ and $\tanh$ are slope-restricted on $[0,1]$ and satisfy the following incremental quadratic constraint \citep{fazlyab2019efficient,pauli2021training}\footnote{Note that \citet{fazlyab2019efficient} suggest using full-block multipliers $\Lambda$, however this construction is incorrect as corrected by \citet{pauli2021training}.}.

\begin{lemma}[Slope-restriction ] \label{lem:sloperestriction}
    Suppose  $\sigma:\bbR\to\bbR$ is slope-restricted on $[0,1]$. Then for all $\Lambda\in\bbD_{++}^n$, the vector-valued function $\sigma(u)^\top=\begin{bmatrix}
        \sigma(u_1) & \dots & \sigma(u_n)
    \end{bmatrix}:\bbR^n\to\bbR^n$ satisfies
\begin{equation}\label{eq:slope_restriction}
    \begin{bmatrix}
        \Delta u\\
        \Delta y
    \end{bmatrix}^\top
    \begin{bmatrix}
        0 & \Lambda\\
        \Lambda & -2\Lambda
    \end{bmatrix}
    \begin{bmatrix}
        \Delta u\\
        \Delta y
    \end{bmatrix}\geq 0 \quad \forall u^a,u^b\in \bbR^n,
\end{equation}
where $\Delta u = u^a-u^b$ and $\Delta y = \sigma(u^a)-\sigma(u^b)$.
\end{lemma}

\subsection{Layer-wise LMI conditions}\label{sec:layerwise_LMIs}

Using the quadratic constraints \eqref{eq:slope_restriction} to over-approximate the nonlinear activation functions, \citep{fazlyab2019efficient,gramlich2023convolutional,pauli2023lipschitz-ana,pauli2024lipschitz} formulate SDPs for Lipschitz constant estimation. The works \citep{pauli2023lipschitz-ana,pauli2024lipschitz} derive layer-wise LMI conditions for 1-D and 2-D CNNs, respectively. In this work, we characterize Lipschitz NNs by these LMIs, thus addressing Problem~\ref{problem-ana}. More specifically, the LMIs in \citep{pauli2023lipschitz-ana,pauli2024lipschitz} yield incrementally dissipative layers and, as a result, the end-to-end mapping satisfies \eqref{eq:lipschitz}, as detailed next in Theorem \ref{thm:certification}.

\begin{theorem}\label{thm:certification}
    Let every layer $k=1,\dots,l$ of an NN \eqref{eq:NN}, \eqref{eq:inputOutputNN} be incrementally dissipative with respect to the supply \eqref{eq:supply} and let $X_0=R$, $X_l=Q$. Then the input-output mapping $u_1\mapsto y_l$ satisfies \eqref{eq:lipschitz}. 
\end{theorem}
\begin{pf}
    All layers $k=1,\dots,l$ are incrementally dissipative with respect to the supply \eqref{eq:supply}, i.e.,
    \begin{equation}\label{eq:supply_pf}
        \|\Delta u_{k}\|_{X_{k-1}}-\|\Delta y_{k}\|_{X_k}\geq 0,\quad k=1,\dots,l.
    \end{equation}
    We sum up \eqref{eq:supply_pf} for all $k=1,\dots,l$ layers and insert $X_0=R$, $X_l=Q$. 
    This yields
    \begin{align}\label{eq:proof1}
        &\Vert \Delta u_1 \Vert_R^2 - \Vert \Delta y_1 \Vert_{X_{1}}^2 + \Vert \Delta u_2 \Vert_{X_{1}}^2 - \dots\\\nonumber
        &-\Vert \Delta y_{l-1} \Vert_{X_{l-1}}^2 + \Vert \Delta u_l \Vert_{X_{l-1}}^2 - \Vert \Delta y_l \Vert_Q^2  \geq 0.
    \end{align}
    Using $u_{k+1}=y_k$, cmp. \eqref{eq:inputOutputNN}, we recognize that \eqref{eq:proof1} entails a telescoping sum. We are left with $\Vert \Delta u_1 \Vert_R^2 - \Vert \Delta y_l \Vert_Q^2  \geq 0$.
\end{pf}

Note that at a layer transition the directional gain matrix $X_k$ is shared between the current and the subsequent layer, which is a natural consequence of the LMI derivation in \citep{pauli2024lipschitz} and accounts for the feedforward interconnection of the NN. During training, the parameters $\theta_k$ are learned. 
Activation function layers and pooling layers typically do not hold any parameters $\theta_k$ and it is convenient to combine fully connected layers and the subsequent activation function layer $\sigma\circ\calF$ and convolutional layers and the subsequent activation function layer  $\sigma\circ\calC$, or even a convolutional layer, an activation function and a pooling layer $\calP\circ\sigma\circ\calC$ and treat these concatenations as one layer. In this way, we split the CNN into subnetworks, each holding parameters $\theta_k$ to be learned. Previous approaches parameterize all convolutional and fully connected layers as 1-Lipschitz and leverage the fact that pooling layers and activation functions are Lipschitz by design. By choosing an LMI-based approach that includes pooling layers and activation functions in layer concatenations rather than using 1-Lipschitz linear layers, we account for the coupling of information between neurons. This results in better expressivity. In the following, we state LMIs that imply incremental dissipativity with respect to \eqref{eq:supply} for the layer types $\sigma\circ\calF$, $\sigma\circ\calC$, and $\calP\circ\sigma\circ\calC$.

\paragraph*{Convolutional layers}
\begin{lemma}[LMI for $\sigma\circ\calC$ \citep{pauli2024lipschitz}]\label{lem:conv+act}
    Consider a 2-D (1-D) convolutional layer $\sigma\circ\calC$ with activation functions that are slope-restricted in $[0,1]$. For some $X \in \bbS_{++}^{c}$ and $X_{-} \in \bbS_{++}^{c_{-}}$, $\sigma\circ\calC$ satisfies \eqref{eq:dissipativity} with respect to the supply \eqref{eq:supply}  if there exist positive definite matrices $P_1 \in \bbS_{++}^{n_{1}}$, $P_2 \in \bbS_{++}^{n_{2}}$, $\bP = \blkdiag (P_1, P_2)$ ($\bP \in \bbS_{++}^{n}$) and a diagonal matrix $\Lambda\in\bbD_{++}^{c}$ such that
    \begin{equation}\label{eq:cert_conv+act}
        \left[\begin{array}{cc|c}
            \bP-\bA^\top \bP\bA  & -\bA^\top \bP\bB & -\bC^\top\Lambda\\
            -\bB^\top \bP\bA &  X_{-}-\bB^\top \bP\bB  & -\bD^\top\Lambda\\\hline
            -\Lambda \bC & -\Lambda \bD & 2\Lambda-X
        \end{array}\right]\succeq 0.
    \end{equation}
\end{lemma}

\begin{pf}
    The proof follows typical arguments used in robust dissipativity proofs, using Lemma \ref{lem:sloperestriction}, i.e., exploiting the slope-restriction property of the activation functions. The proof is provided in \citep{pauli2024lipschitz}.
\end{pf}

\begin{corollary}[LMI for $\calP\circ\sigma\circ\calC$]\label{cor:conv+act+pool}
    Consider  a 2-D (1-D) convolutional layer $\calP\circ\sigma\circ\calC$ with activation functions that are slope-restricted in $[0,1]$ and an average pooling layer / a maximum pooling layer. For some $X \in \bbS_{++}^{c}$ / $X \in \bbD_{++}^{c}$  and $X_{-} \in \bbS_{++}^{c_{-}}$, $\calP\circ\sigma\circ\calC$ satisfies \eqref{eq:dissipativity} with respect to supply \eqref{eq:supply} if there exist positive definite matrices $P_1 \in \bbS_{++}^{n_{1}}$, $P_2 \in \bbS_{++}^{n_{2}}$, $\bP = \blkdiag (P_1, P_2)$ ($\bP \in \bbS_{++}^{n}$) and a diagonal matrix $\Lambda\in\bbD_{++}^{c}$ such that
    \begin{equation}\label{eq:cert_conv+act+pool}
        \left[\begin{array}{cc|c}
            \bP-\bA^\top \bP\bA  & -\bA^\top \bP\bB & -\bC^\top\Lambda\\
            -\bB^\top \bP\bA &  X_{-}-\bB^\top \bP\bB  & -\bD^\top\Lambda\\\hline
            -\Lambda \bC & -\Lambda \bD & 2\Lambda-\rho_{\mathrm{p}}^2X
        \end{array}\right] \succeq 0,
    \end{equation}
    where $\rho_{\mathrm{p}}$ is the Lipschitz constant of the pooling layer.
\end{corollary}

\begin{rem}
    Lemma \ref{lem:conv+act} and Corollary \ref{cor:conv+act+pool} entail all kinds of zero-padding \citep{pauli2024lipschitz}, just like \citep{prach2022almost}, giving our method an advantage over \citep{trockman2021orthogonalizing,wang2023direct}, which are restricted to circular padding.
\end{rem}

\paragraph*{Fully connected layers}
\begin{lemma}[LMI for $\sigma\circ\calF$ \citep{pauli2024lipschitz}]\label{lem:FC+act}
    Consider a fully connected layer $\sigma\circ\calF$ with activation functions that are slope-restricted in $[0,1]$. For some $X \in \bbS_{++}^{c}$ and $X_{-} \in \bbS_{++}^{c_{-}}$, $\sigma\circ\calF$ satisfies \eqref{eq:dissipativity} with respect to \eqref{eq:supply} if there exists a diagonal matrix $\Lambda\in\bbD_{++}^{c}$ such that
    \begin{equation}\label{eq:cert_FC+act}
        \begin{bmatrix}
            X_{-}  & -W^\top\Lambda\\
            -\Lambda W & 2\Lambda-X
        \end{bmatrix}\succeq 0.
    \end{equation}
\end{lemma}

\begin{rem}
    Technically, we can interpret a fully connected layer as a 0-D convolutional layer with a signal length of 1, $\bD=W$ and $\bA=0$, $\bB=0$, $\bC=0$. Accordingly, \eqref{eq:cert_FC+act} is a special case of \eqref{eq:cert_conv+act}.
\end{rem}

\paragraph*{The last layer}
The last layer is treated separately, as it typically lacks an activation function and $X_l=Q$ is predefined. In classifying NNs the last layer typically is a fully connected layer $\calF_l$, for which the LMI
    \begin{equation}\label{eq:cert_last_FC}
        X_{l-1} - W_l^\top Q W_l \succeq 0
    \end{equation}
implies \eqref{eq:dissipativity} with respect to the supply \eqref{eq:supply}, cmp. Theorem \ref{thm:certification}.

We denote the LMIs \eqref{eq:cert_conv+act} to \eqref{eq:cert_last_FC} as instances of $\calG_k(X_k, X_{k-1}, \nu_k) \succeq 0$, where $\nu_k$ denote the respective multipliers and slack variables in the specific LMIs (for $\sigma \circ \calF_k$, $\nu_k=\Lambda_k$, for $\sigma \circ \calC_k$, $\nu_k=(\Lambda_k,\bP_k)$). 

\begin{rem}[Lipschitz constant estimation]
To determine an upper bound on the Lipschitz constant for a given NN, we solve the SDP
\begin{align}\label{eq:Lipschitz_estimation}
    \begin{split}
    \min_{\rho^2,X,\nu} ~\rho^2 ~ \text{s.\,t.} ~ \calG_k(X_k,X_{k-1},\nu_k)\succeq 0, ~k=1,\dots,l\\
    X_0=R=\rho^2 I,~X_l=Q= I.
    \end{split}
\end{align}
In \eqref{eq:Lipschitz_estimation}, $X=\{X_k\}_{k=1}^{l-1}$, $\nu=\{\nu_k\}_{k=1}^{l}$, $\rho^2$ serve as decision variables. Based on Theorem \ref{thm:certification}, the solution for $\rho$ is an upper bound on the Lipschitz constant for the NN \citep{pauli2024lipschitz}.
\end{rem}

\section{Synthesis of Dissipative Layers}\label{sec:parameterization}
In the previous section, we revisited LMIs, derived in \citep{pauli2024lipschitz} for Lipschitz constant estimation for NNs, which we use to characterize robust NNs that satisfy \eqref{eq:lipschitz} or \eqref{eq:lipschitz_2}. This work is devoted to the synthesis of such Lipschitz-bounded NNs. To this end, in this section, we derive layer-wise parameterizations for $\theta_k$ that render the layer-wise LMIs $\calG_k(X_k,X_{k-1},\nu_k)\succeq 0$, $k=1,\dots,l$ feasible by design, addressing Problem~\ref{problem1}. For our parameterization the Lipschitz bound $\rho$ or, respectively, the matrices $Q$, $R$ are hyperparameters that can be chosen by the user. Low Lipschitz bounds $\rho$ lead to high robustness, yet compromise the expressivity of the NN, as we will observe in Subsection~\ref{sec:exp_comparison}. Inserting the parameterizations $\phi\mapsto\theta$ presented in this section into \eqref{eq:opt_con} yields \eqref{eq:opt}, which can be conveniently solved using first-order solvers.

After introducing the Cayley transform in Subsection~\ref{sec:Cayley}, we discuss the parameterization of fully connected layers and convolutional layers in Subsections~\ref{sec:parameterization_FC} and \ref{sec:parameterization_conv}, respectively, based on the Cayley transform and a solution to the 1-D and 2-D Lyapunov equations. To improve readability, we drop the layer index $k$ in this section. If we refer to a variable of the previous layer, we denote it by the subscript ``$-$''.

\subsection{Cayley transform}\label{sec:Cayley}
The Cayley transform maps skew-symmetric matrices to orthogonal matrices, and its extended version parameterizes the Stiefel manifold from non-square matrices. The Cayley transform can be used to map continuous time systems to discrete time systems \citep{guo2006relation}. Furthermore, it has proven useful in designing NNs  with norm-constrained weights or Lipschitz constraints \citep{trockman2021orthogonalizing,helfrich2018orthogonal,wang2023direct}. 

\begin{lemma}[Cayley transform]\label{lem:Cayley}
    For all $Y\in\bbR^{n\times n}$ and $Z\in\bbR^{m\times n}$ the Cayley transform
\begin{equation*}
    \begin{split}
        \Cayley\left(
    \begin{bmatrix}
        Y \\
        Z
    \end{bmatrix}\right)=
    \begin{bmatrix}
        U \\
        V
    \end{bmatrix}=
    \begin{bmatrix}
        (I+M)^{-1}(I-M) \\
        2Z(I+M)^{-1}
    \end{bmatrix},
    \end{split}
\end{equation*}
where $M=Y-Y^\top+Z^\top Z$, yields matrices $U\in\bbR^{n\times n}$ and $V\in\bbR^{m\times n}$ that satisfy $U^\top U + V^\top V =  I$.
\end{lemma}

Note that $I+M$ is nonsingular since $1\leq\lambda_\mathrm{min}(I+Z^\top Z)\leq Re(\lambda_\mathrm{min}(I+M))$.

\subsection{Fully connected layers}\label{sec:parameterization_FC}
For fully connected layers $\sigma\circ\calF$, Theorem \ref{thm:FC} gives a mapping $\phi\mapsto(W,b)$ from unconstrained variables $\phi$ that renders \eqref{eq:cert_FC+act} feasible by design, and thus the layer is dissipative with respect to the supply \eqref{eq:supply}. 
\begin{theorem}\label{thm:FC}
    A fully connected layer $\sigma\circ\calF$ parameterized by
    \begin{align}\label{eq:parameterization_FNN}
        W =  \sqrt{2}\Gamma^{-1}V^\top L_{-},
    \end{align}
wherein
    \begin{equation*}
        \Gamma=\diag(\gamma),~
        \begin{bmatrix}
            U \\
            V
        \end{bmatrix}=
        \Cayley\left(
        \begin{bmatrix}
            Y\\
            Z
        \end{bmatrix}
        \right),~
        L=\sqrt{2} U\Gamma,
    \end{equation*}
    satisfies \eqref{eq:cert_FC+act}. This yields the mapping
    \begin{equation*}
        (Y,Z,\gamma,b)\mapsto(W,b,L),    
    \end{equation*}
    where $Y\in\bbR^{c\times c}$, $Z\in\bbR^{c_-\times c}$, $\gamma,b\in\bbR^{c}$.
\end{theorem}
A proof is provided in \cite[Theorem 5]{pauli2023lipschitz-syn}. 
We collect the free variables in $\phi = (Y,Z,\gamma,b)$ and the weight and bias terms in $\theta=(W,b)$. To train a Lipschitz-bounded NN, we parameterize the weights $W$ of all fully connected layers using \eqref{eq:parameterization_FNN} and then train over the free variables $\phi$ using \eqref{eq:opt}. Toolboxes are used to determine gradients with respect to $\phi$. The by-product $\Gamma$ parameterizes $\Lambda=\Gamma^2$, and $L$ parameterizes the directional gain $X=L^\top L$ and is passed on to the subsequent layer, where it appears as $X_-$. The first layer $k=1$ takes $ L_0=\chol(R)$, which is $L_0=\rho I$ when considering \eqref{eq:lipschitz_2}. Incremental properties such as Lipschitz boundedness are independent of the bias term such that $b\in\bbR^{c}$ is a free variable as well.

Note that the parameterization \eqref{eq:parameterization_FNN} for fully connected layers of a Lipschitz-bounded NN is the same as the one proposed in \citep{wang2023direct}. According to \citep[Theorem 3.1]{wang2023direct}, \eqref{eq:parameterization_FNN} is necessary and sufficient, i.\,e., the fully connected layers $\sigma\circ\calF$ satisfy \eqref{eq:cert_FC+act} if and only if the weights can be parameterized by \eqref{eq:parameterization_FNN}.

\begin{rem}
    To ensure that $\Gamma$ and $L$ are nonsingular, w.l.o.g., we may parameterize $\Gamma=\diag(e^{\gamma})$, $\gamma\in\bbR^c$ \citep{wang2023direct} and $L=U^\top \diag(e^{l})V$, $l\in\bbR^c$ with square orthogonal matrices $U$ and $V$, e.g., found using the Cayley transform. 
\end{rem}

\subsection{Convolutional layers}\label{sec:parameterization_conv}
The parameterization of convolutional layers is divided into two steps. We first parameterize the upper left block in \eqref{eq:cert_conv+act}, namely
\begin{equation}\label{eq:matrixF}
F \coloneqq
\left[\begin{array}{cc}
    \bP-\bA^\top \bP\bA  & -\bA^\top \bP\bB\\
    -\bB^\top \bP\bA &  X_{-}-\bB^\top \bP\bB 
\end{array}\right]\succ 0,
\end{equation}
by constructing a parameter-dependent solution of a 1-D or 2-D Lyapunov equation. Secondly, we parameterize $\bC$ and $\bD$ from the auxiliary variables determined in the previous step.

To simplify the notation of \eqref{eq:cert_conv+act} to
\begin{equation*}
\begin{bmatrix}
    F & -\widehat{\bC}^\top\Lambda\\
    -\Lambda \widehat{\bC} & 2\Lambda-X
\end{bmatrix}\succeq 0,
\end{equation*}
we introduce $\widehat{\bC} \coloneqq
        \begin{bmatrix}
            \bC & \bD
        \end{bmatrix}$. In the following, we distinguish between the parameterization of 1-D convolutional layers and 2-D convolutional layers.

\paragraph*{1-D convolutional layers}
The parameterization of 1-D convolutional layers uses the controllability Gramian \citep{pauli2023lipschitz-syn}, which is the unique solution to a discrete-time Lyapunov equation. The first parameterization step entails to parameterize $\bP$ such that \eqref{eq:matrixF} is feasible. To do so, we use the following lemma \citep{pauli2023lipschitz-syn}.

\begin{lemma}[Parameterization of $\bP$]\label{lem:Gramian_1D}
Consider the 1-D state space representation \eqref{eq:ABCD}. For some $\varepsilon>0$ and all $H\in\bbR^{n\times n}$, the matrix $\bP=\bT^{-1}$ with
\begin{equation}\label{eq:Gramian}
    \bT = \sum_{k=0}^{n-c_{-}} \bA^k(\bB{X}_{-}^{-1}\bB^\top+H^\top H+\epsilon I)(\bA^\top)^{k},
\end{equation}
renders \eqref{eq:matrixF} feasible.
\end{lemma}

A proof is provided in \cite[Lemma 7]{pauli2023lipschitz-syn}. The key idea behind the proof is that by Schur complements \eqref{eq:matrixF} can be posed as a Lyapunov equation. The expression \eqref{eq:Gramian} then provides the solution to this Lyapunov equation. The second step now parameterizes $\widehat{\bC}$ from $F$, as detailed in Theorem~\ref{thm:conv_1D}. All kernel parameters appear in $\widehat{\bC}$, cmp. the chosen state space represenation \eqref{eq:ABCD}, and are parameterized as follows. 

\begin{theorem}\label{thm:conv_1D}
    A 1-D convolutional layer $\sigma\circ\calC$ that is parameterized by
    \begin{equation}\label{eq:parameterization_CNN}
        \widehat{\bC} = \begin{bmatrix} \bC & \bD \end{bmatrix}
        =  \sqrt{2}\Gamma^{-1}V^\top L_F,
    \end{equation}
wherein
\begin{equation*}
    \Gamma=\diag(\gamma),~
    \begin{bmatrix}
        U \\
        V
    \end{bmatrix}=
    \Cayley\left(
    \begin{bmatrix}
        Y\\
        Z
    \end{bmatrix}
    \right),~
    L_F=\chol(F),
\end{equation*}
satisfies \eqref{eq:cert_conv+act}. Here, $F$ is given by \eqref{eq:matrixF} with $\bP$ parameterized from $X_{-}$ and $H$ using \eqref{eq:Gramian}, where $X=L^\top L,~L_0 =R,~L=\sqrt{2} U\Gamma$. This yields the mapping
\begin{equation*}
    (Y,Z,H,\gamma,b)\mapsto(K,b,L)
\end{equation*}
with $Y\in\bbR^{c\times c}$, $Z\in\bbR^{(r+1) c_{-}\times c}$, $H\in\bbR^{n\times n}$, $\gamma, b\in\bbR^c$.
\end{theorem}

Note that we have to slightly modify $L$ in case the convolutional layer contains an average pooling layer. We then parameterize $L = \frac{\sqrt{2}}{\rho_\mathrm{p}} U \Gamma$, where $\rho_\mathrm{p}$ is the Lipschitz constant of the average pooling layer. In case the convolutional layer contains a maximum pooling layer, i.\,e., $\calP^\mathrm{max}\circ \sigma\circ \calC$, we need to modify the parameterization of $L$ to ensure that $X$ is a diagonal matrix, cmp. Corollary~\ref{cor:conv+act+pool}.

\begin{corollary}\label{cor:maxpooling}
    A 1-D convolutional layer that contains a maximum pooling layer $\calP^{\mathrm{max}}\circ\sigma\circ\calC$ parameterized by
    \begin{equation}\label{eq:parameterization_CNN2}
        \widehat{\bC} = \begin{bmatrix} \bC & \bD \end{bmatrix} = \Lambda^{-1}\widetilde{\Gamma}\widetilde{U}^\top L_F,
    \end{equation}
wherein
\begin{align*}
    \Lambda = \frac{1}{2} \left(\widetilde{\Gamma}^\top\widetilde{\Gamma}+Q\right),~
    \widetilde{\Gamma}\!=\!\diag(\tilde{\gamma}),~
    \widetilde{U}=\Cayley(\widetilde{Y}),
\end{align*}
satisfies \eqref{eq:cert_conv+act}. Here, $F$ is given by \eqref{eq:matrixF} with $\bP$ parameterized from $X_-$ and $H$ using \eqref{eq:Gramian}, where $X=L^\top L,~L_0 =\rho I,~L =\diag(l)$, $L_F=\chol(F)$. The free variables $\widetilde{Y}\in\bbR^{rc_-\times c}$, $H\in\bbR^{n\times n}$, $\tilde{\gamma},l\in\bbR^{c}$, compose the mapping
\begin{equation*}
(\widetilde{Y},H,\tilde{\gamma},l,b)\mapsto(K,b, L).
\end{equation*}
\end{corollary}
Proofs of Theorem \ref{thm:conv_1D} and Corollary \ref{cor:maxpooling} are provided in \cite[Theorem 8 and Corollary 9]{pauli2023lipschitz-syn}.

\paragraph*{2-D convolutional layers}
Next, we turn to the more involved case of 2-D convolutional layers \eqref{eq:conv_2D}. The parameterization of 2-D convolutional layers in their 2-D state space representation, i.e., the direct parameterization of the kernel parameters, is one of the main technical contributions of this work. Since there does not exist a solution for the 2-D Lyapunov equation in general \citep{anderson1986stability}, we construct one for the special case of a 2-D convolutional layer, which is a 2-D FIR filter. The utilized state space representation of the FIR filter \eqref{eq:ABCD_2D} has a characteristic structure, which we leverage to find a parameterization.

We proceed in the same way as in the 1-D case by first parameterizing $\bP$ to render \eqref{eq:matrixF} feasible. In the 2-D case this step requires to consecutively parameterize $P_1$ and $P_2$ that make up $\bP=\blkdiag(P_1,P_2)$. Inserting \eqref{eq:ABCD_2D} into \eqref{eq:RoesserSys}, we recognize that the $x_2$ dynamic is decoupled from the $x_1$ dynamic due to $A_{21}=0$. Consequently, $P_2$ can be parameterized in a first step, followed by the parameterization of $P_1$. Let us define some auxiliary matrices $T_1=P_1^{-1}$, $T_2=P_2^{-1}$, $\bT=\blkdiag(T_1,T_2)$,
\begin{equation}\label{eq:Qtilde}
    \widetilde{\bX} =
    \begin{bmatrix}
        \widetilde{X}_{11} & \widetilde{X}_{12}\\
        \widetilde{X}_{12}^\top & \widetilde{X}_{22}
    \end{bmatrix} =
    \bB X_-^{-1} \bB^\top,
\end{equation}
which is partitioned according to the state dimensions $n_1$ and $n_2$, i.e., $\widetilde{X}_{11}\in\bbR^{n_1\times n_1}$, $\widetilde{X}_{12}\in\bbR^{n_1\times n_2}$, $\widetilde{X}_{22}\in\bbR^{n_2\times n_2}$. We further define 
\begin{align}\label{eq:Q11hat}
    \begin{split}
    \widehat{X}_{11}=&A_{12}T_2A_{12}^\top + \widetilde{X}_{11} +(\widetilde{X}_{12}+A_{12}T_2A_{22}^\top)\\
    &(T_2-A_{22}T_2A_{22}^\top-\widetilde{X}_{22})^{-1}(\widetilde{X}_{12}+A_{12}T_2A_{22}^\top)^\top.
    \end{split}
\end{align}

\begin{lemma}\label{lem:Gramian_2D}
Consider the 2-D state space representation \eqref{eq:ABCD_2D}. For some $\varepsilon>0$ and all $H_1\in\bbR^{n_{1}\times n_{1}}$, $H_2\in\bbR^{n_{2}\times n_{2}}$, the matrices $P_1=T_1^{-1}$, $P_2=T_2^{-1}$  with
\begin{subequations}\label{eq:Gramians_2D}
\begin{equation}
    T_1 = \sum_{k=0}^{n_{1}-c} A_{11}^k(\widehat{X}_{11}+H_1^\top H_1+\epsilon I)(A_{11}^\top)^{k},\label{eq:Gramian_X1}
\end{equation}
\begin{equation}
    T_2 = \sum_{k=0}^{n_{2}-c_-} A_{22}^k(\widetilde{X}_{22}+H_2^\top H_2+\epsilon I)(A_{22}^\top)^{k}\label{eq:Gramian_X2}
\end{equation}
\end{subequations}
render \eqref{eq:matrixF} feasible.
\end{lemma}
\begin{pf}
    Let us first consider the parameterization of $T_2$. Given that $A_{22}$ is a nilpotent matrix, cmp. \eqref{eq:ABCD_2D}, \eqref{eq:Gramian_X2} is equivalent to
    \begin{equation*}
        T_2 = \sum_{k=0}^{\infty} A_{22}^k(\widetilde{X}_{22}+H_2^\top H_2+\epsilon I)(A_{22}^\top)^{k},     
    \end{equation*}
    which in turn is the unique solution to the Lyapunov equation
    \begin{equation}\label{eq:Lyapunov_equation_T2}
        T_2-A_{22}T_2A_{22}^\top-\widetilde{X}_{22}=H_2^\top H_2+\epsilon I\succ 0
    \end{equation}
    by \cite[Theorem 6.D1]{chen1984linear}. Next, we utilize that \eqref{eq:Gramian_X1} is equivalent to
    \begin{equation}\label{eq:Gramian_X1_infty}
        T_1 = \sum_{k=0}^{\infty} A_{11}^k(\widehat{X}_{11}+H_1^\top H_1+\epsilon I)(A_{11}^\top)^{k},
    \end{equation}
    due to the fact that $A_{11}$ is also nilpotent. Equation \eqref{eq:Gramian_X1_infty} in turn is the unique solution to the Lyapunov equation
    \begin{equation*}
        T_1-A_{11}T_1A_{11}^\top-\widehat{X}_{11}=H_1^\top H_1+\epsilon I\succ 0
    \end{equation*}
    by \cite[Theorem 6.D1]{chen1984linear}. Using the definition \eqref{eq:Q11hat}, wherein the term $T_2-A_{22}T_2A_{22}^\top-\widetilde{X}_{22}\succ 0$ according to \eqref{eq:Lyapunov_equation_T2}, we apply the Schur complement to $T_1-A_{11}T_1A_{11}^\top-\widehat{X}_{11}\succ 0$. We obtain
    \begin{equation}\label{eq:Lyapunov2}
    \begin{bmatrix}
        T_1 & 0\\
        0   & T_2
    \end{bmatrix}
    -
    \begin{bmatrix}
        A_{11} & A_{12}\\
        0   & A_{22}
    \end{bmatrix}
    \begin{bmatrix}
        T_1 & 0\\
        0   & T_2
    \end{bmatrix}
    \begin{bmatrix}
        A_{11}^\top   & 0\\
        A_{12}^\top   & A_{22}^\top
    \end{bmatrix}
    -
    \widetilde{\bX}\succ 0,
\end{equation}
which can equivalently be written as $\bT-\bA\bT\bA^\top-\bB(X_-)^{-1}\bB^\top\succ 0$ using \eqref{eq:Qtilde}, to which we again apply the Schur complement. This yields
\begin{equation}\label{eq:Schurcomplemnt_Gramian}
\begin{bmatrix}
    \bT^{-1}  & 0         & \bA^\top\\
    0         & X_-   & \bB^\top \\
    \bA       & \bB       & \bT
\end{bmatrix}\succ 0.    
\end{equation}
Finally, we again apply the Schur complement to \eqref{eq:Schurcomplemnt_Gramian} with respect to the lower right block and replace $\bP=\bT^{-1}$, which results in $F\succ 0$.
\end{pf}

Note that the parameterization of $\bT$ takes the free variables $H_1$, $H_2$, $A_{12}$, and $B_1$. The matrices $A_{11}$, $A_{22}$, and $B_2$ are predefined by the chosen state space representation \eqref{eq:ABCD_2D}.

\begin{rem}
In the case of strided convolutional layers with $\bs\geq 2$, $A_{12}$ and $B_1$ may also have a predefined structure and  zero entries, see Remark \ref{rem:ss_stride} and \citep{pauli2024state}, which we can incorporate into the parameterization, as well.
\end{rem}

For the second part of the parameterization, we partition \eqref{eq:cert_conv+act} as
\begin{equation}\label{eq:LMI_partitionF}
    \begin{bmatrix}
        F_1 & F_{12} & -C_1^\top \Lambda\\
        F_{12}^\top & F_2 & -\widehat{C}_2^\top \Lambda\\
        -\Lambda C_1 & -\Lambda \widehat{C}_2 & 2\Lambda -X
    \end{bmatrix}\succeq 0,
\end{equation}
and define $\widehat{C}_2=\begin{bmatrix} C_2 & D \end{bmatrix}$, noting that $\widehat{C}_2$ holds all parameters of $K$ that are left to be parameterized, cmp. Lemma~\ref{lem:min_real_2D}. Next, we introduce Lemma~\ref{lem:diagonal_dominance} that we used to parameterize convolutional layers, directly followed by Theorem~\ref{thm:conv} that states the parameterization.

\begin{lemma}[Theorem 3 \citep{araujoICLR2023}]\label{lem:diagonal_dominance}
Let $W\in\bbR^{m\times n}$ and $T\in\bbD_{++}^{n}$. If there exists some $Q\in\bbD_{++}^n$ such that $T-QW^\top WQ^{-1}$ is a symmetric and real diagonally dominant matrix, i.e.,
\begin{equation*}
    \vert Q_{ii} \vert > \sum_{j=1,i\neq j} \vert Q_{ij}\vert,\quad \forall i=1,\dots,n,
\end{equation*}
then $T\succ W^\top W$.
\end{lemma}

\begin{theorem}\label{thm:conv}
A 2-D convolutional layer $\sigma\circ\calC$ parameterized by
    \begin{align*}
    	\begin{bmatrix}
    		C_2 & D
    	\end{bmatrix}=
            \widehat{C}_2 &=  C_1F_1^{-1}F_{12} - L_{\Gamma}^{\top}V^\top L_F,
    \end{align*}
wherein for some $\epsilon>0$
    \begin{align*}
        &L_{\Gamma}=\chol(2\Gamma - C_1 F_{1}^{-1} C_1^\top),\quad
        \Gamma=\diag(\gamma),\\
        &\gamma_i=\epsilon+\delta_i^2+\sum_j \frac{1}{2}\left| C_1 F_1^{-1} C_1^\top \right|_{ij}\frac{q_j}{q_i},~i=1,\dots,c, \\
        &\begin{bmatrix}
            U \\
            V
        \end{bmatrix}=
        \Cayley\left(
        \begin{bmatrix}
            Y\\
            Z
        \end{bmatrix}
        \right),~
        L_F=\chol(F_2 -F_{12}^\top F_{1}^{-1} F_{12}),
    \end{align*}
satisfies \eqref{eq:cert_conv+act}. Here, $F$ is parameterized from $X_-$ and free variables $H_1$, $H_2$, $B_1$, $A_{12}$ using \eqref{eq:Gramians_2D}, where $X=L^\top L,~L_0 =\rho I,~L= UL_\Gamma\Gamma^{-1}$.
This yields the mapping
\begin{equation*}
    (Y,Z,H_1,H_2,A_{12},B_1,\delta,q,b) \mapsto (K,b,L),
\end{equation*}
where $Y\in\bbR^{c\times c}$, $Z\in\bbR^{c_{-}\times c}$, $H_1\in\bbR^{n_{1}\times n_{1}}$, $H_2\in\bbR^{n_{2}\times n_{2}}$, $A_{12}\in\bbR^{n_{1}\times n_{2}}$, $B_1\in\bbR^{n_{1}\times c_-}$,  $\delta,q,b\in\bbR^{c}$. 
\end{theorem}

\begin{pf}
    The matrices $U$ and $V$ are parametrized by the Cayley transform such that they satisfy $U^\top U+V^\top V=I$. We solve for $U=L\Gamma L_\Gamma^{-1}$ and $V=L_F^{-\top}(-\widehat{C}_2 +C_1 F_1^{-1}F_{12})^\top L_\Gamma^{-1}$ and replace $L_F$ with its definition, which we then insert into $U^\top U +V^\top V=I$, yielding
    \begin{align*}
        &L_\Gamma^{-\top} \Gamma X\Gamma L_\Gamma^{-1}+L_\Gamma^{-\top} (-\widehat{C}_2 + C_1 F_1^{-1}F_{12})\\
        &(F_2-F_{12}^\top F_1^{-1}F_{12})^{-1}(-\widehat{C}_2 +C_1 F_1^{-1}F_{12})^\top L_\Gamma^{-1}=I.
    \end{align*}
    By left and right multiplication of this equation with $L_\Gamma^{\top}$ and $L_\Gamma$, respectively, we obtain
    \begin{align}\label{eq:proof}
        \begin{split}
        \Gamma X\Gamma+(-\widehat{C}_2 + C_1 F_1^{-1}F_{12})(F_2-F_{12}^\top F_1^{-1}F_{12})^{-1}\\(-\widehat{C}_2 +C_1 F_1^{-1}F_{12})^\top=2\Gamma-C_1F_1^{-1}C_1^\top.
        \end{split}
    \end{align}
    We next show that $2\Gamma-C_1F_1^{-1}C_1^\top$ is positive definite and therefore admits a Cholesky decomposition. Since $F_1 \succ 0$, $ C_1F_1^{-1}C_1^\top\succeq 0$ such that we know that $0\leq( C_1F_1^{-1}C_1^\top)_{ii}=| C_1F_1^{-1}C_1^\top|_{ii}$. With this, we notice that $2\Gamma-Q C_1F_1^{-1}C_1^\top Q^{-1}$ with $Q=\diag(q)$ is diagonally dominant as it component-wise satisfies
    \begin{align*}
    &2\epsilon+2\delta_i^2+ \sum_{j}| C_1F_1^{-1}C_1^\top|_{ij} \frac{q_j}{q_i} - ( C_1F_1^{-1}C_1^\top)_{ii} \frac{q_i}{q_i} \\
    &> \sum_{j,i\neq j}| C_1F_1^{-1}C_1^\top|_{ij} \frac{q_j}{q_i}\quad \forall i=1,\dots,n.
    \end{align*}
    We see that diagonal dominance holds using that 
    \begin{equation*}
    \sum_{j}| C_1F_1^{-1}C_1^\top|_{ij}=| C_1F_1^{-1}C_1^\top|_{ii}+\sum_{j,i\neq j}| C_1F_1^{-1}C_1^\top|_{ij},
    \end{equation*}
    which yields $2\epsilon+2\delta_i^2>0$, which in turn holds trivially.  According to Lemma \ref{lem:diagonal_dominance}, the fact that $2\Gamma-Q C_1F_1^{-1}C_1^\top Q^{-1}$ is diagonally dominant implies that $2\Gamma-C_1F_1^{-1}C_1^\top$ is positive definite.
    
    Equality \eqref{eq:proof} implies the inequality
    \begin{align*}
        \begin{split}
         2\Gamma-C_1F_1^{-1}C_1^\top-\Gamma X\Gamma-(-\widehat{C}_2 + C_1 F_1^{-1}F_{12})\\
         (F_2-F_{12}^\top F_1^{-1}F_{12})^{-1}(-\widehat{C}_2 +C_1 F_1^{-1}F_{12})^\top\succeq 0,
        \end{split}
    \end{align*}
    which we left and right multiply with $\Lambda=\Gamma^{-1}$, which is invertible as $\gamma_i\geq \epsilon$. We obtain
    \begin{equation}\label{eq:proof2}
        \begin{split}
         2\Lambda-\Lambda C_1F_1^{-1}C_1^\top\Lambda-X-(-\Lambda\widehat{C}_2 + \Lambda C_1 F_1^{-1}F_{12})\\
         (F_2-F_{12}^\top F_1^{-1}F_{12})^{-1}(-\widehat{C}_2^\top\Lambda +F_{12}^\top F_1^{-1}C_1^\top\Lambda)\succeq 0.
        \end{split}
    \end{equation}
    Given that $F\succ 0$, we know that $F_1\succ 0$, $F_2\succ 0$ and by the Schur complement $F_2-F_{12}^\top F_1^{-1}F_{12}\succ 0$. By the Schur complement, \eqref{eq:proof2} is equivalent to
    \begin{equation*}\label{eq:proof_LMI}
    \begin{bmatrix}
        F_2 -F_{12}^\top F_{1}^{-1} F_{12} & -\widehat{C}_2^\top \Lambda +F_{12}^\top F_{1}^{-1} C_1^\top \Lambda \\
         -\Lambda \widehat{C}_2 +\Lambda C_1 F_{1}^{-1}F_{12}  & 2\Lambda -X - \Lambda C_1 F_{1}^{-1} C_1^\top\Lambda
    \end{bmatrix}\succeq 0,        
    \end{equation*}
    which in turn is equivalent to \eqref{eq:LMI_partitionF} again using the Schur complement.
\end{pf}

\begin{rem}
    An alternative parameterization of $\gamma$ in Theorem \ref{thm:conv} would be
    \begin{align*}
    \gamma_i = \epsilon + \delta_i^2+ \frac{1}{2}\sum_{j}| C_1F_1^{-1}C_1^\top|_{ij},~i=1,\dots,c,
\end{align*}    
    obtained by setting $\diag(q)=I$. Another alternative is 
    \begin{equation*}
        \gamma_i = \epsilon + \delta_i^2 + \frac{1}{2}\mathrm{max\,eig}(C_1F_1^{-1}C_1^\top),~i=1,\dots,c
    \end{equation*}
    as it also renders 
    \begin{align*}
        &2\Gamma - C_1F_1^{-1}C_1^\top = 2\epsilon I + 2\diag(\delta^2) \\
        &+ \mathrm{max\,eig}(C_1F_1^{-1}C_1^\top)I - C_1F_1^{-1}C_1^\top
    \end{align*}
    positive definite.
\end{rem}

If the convolutional layer contains a pooling layer, we again need to slightly adjust the parameterization. For average pooling layers, we can simply replace $X$ by $\rho_\mathrm{p}^2 X$, yielding $L= \frac{1}{\rho_\mathrm{p}}UL_\Gamma\Gamma^{-1}$ instead of $L= UL_\Gamma\Gamma^{-1}$, where $\rho_\mathrm{p}$ is the Lipschitz constant of the average pooling layer. Maximum pooling layers are nonlinear operators. For that reason, the gain matrix $X$ needs to be further restricted to be a diagonal matrix, cmp. \citep{pauli2023lipschitz-syn}.

\begin{theorem}
A 2-D convolutional layer that includes a maximum pooling layer with Lipschitz constant $\rho_\mathrm{p}$ parameterized by
\begin{equation*}
    \begin{bmatrix}
        C_2 & D
    \end{bmatrix} = \widehat{C}_2 = C_1F_1^{-1}F_{12} - L_\Gamma^\top\widetilde{U}^\top L_F,
\end{equation*}
wherein for some $\epsilon>0$ 
    \begin{align*}
        &L_{\Gamma}=\chol(2\Gamma - \rho_\mathrm{p}^2 \Gamma X\Gamma - C_1 F_{1}^{-1} C_1^\top),~
        \Gamma=\diag(\gamma),\\
        &\gamma_i=\frac{1}{2}\eta_i+\omega_i^2,\quad
        \eta_i=\epsilon+\delta_i^2+\sum_j \left| C_1 F_1^{-1} C_1^\top \right|_{ij} \frac{q_j}{q_i},\\
        &\widetilde{U}=
        \Cayley(\widetilde{Y}),~
        L_F=\chol(F_2 -F_{12}^\top F_{1}^{-1} F_{12}),
    \end{align*}
satisfies \eqref{eq:cert_conv+act+pool}. Here, $F$ is parameterized from $X_-$ and free variables $H_1$, $H_2$, $B_1$, $A_{12}$ using \eqref{eq:Gramians_2D}, where $X=L^\top L,~ L = \diag(l),~l_i=\frac{\sqrt{2\gamma_i-\eta_i}}{\gamma_i\rho_\mathrm{p}}, L_0 = \rho I$. 
This yields the mapping
\begin{equation*}
    (\widetilde{Y},H_1,H_2,A_{12},B_1,\delta,\omega,b) \mapsto (K,b,L),
\end{equation*}
where $\widetilde{Y}\in\bbR^{c_- (r_2+1) \times c}$, $H_1\in\bbR^{n_{1}\times n_{1}}$, $H_2\in\bbR^{n_{2}\times n_{2}}$, $A_{12}\in\bbR^{n_{1}\times n_{2}}$, $B_1\in\bbR^{n_{1}\times c}$, $\delta, \omega, q, b\in\bbR^{c}$. 
\end{theorem}
\begin{pf}
    The proof follows along the lines of the proof of Theorem~\ref{thm:conv}. We solve for $\widetilde{U}=L_F^{-\top}(-\widehat{C}_2+C_1F_1^{-1}F_{12})^\top L_\Gamma^{-1}$, which we then insert into $\widetilde{U}^\top \widetilde{U}=I$ and subsequently left/right multiply with $L_\Gamma^\top$ and $L_\Gamma$, respectively, to obtain
\begin{align}\label{eq:proof_max}
\begin{split}
    L_\Gamma^\top L_\Gamma  \! = \!  2\Gamma - \rho_\mathrm{P}^2\Gamma X\Gamma - C_1 F_{1}^{-1} C_1^\top \! = \! (-\widehat{C}_2 +C_1 F_{1}^{-1}F_{12})\\ (F_2 -F_{12}^\top F_{1}^{-1} F_{12})^{-1} (-\widehat{C}_2 +C_1 F_1^{-1}F_{12})^\top.
\end{split}
\end{align}
Using $X_{ii}=l_i^2=\frac{2\gamma_i-\eta_i}{\gamma_i^2\rho_{\mathrm{P}}^2}$, we notice that $2\Gamma - \rho_\mathrm{P}^2\Gamma X\Gamma - QC_1 F_{1}^{-1} C_1^\top Q^{-1}$ is by design diagonally dominant as it satisfies
\begin{align*}
    &2\gamma_i-2\gamma_i+\eta_i-\left|C_1F_1^{-1}C_1^\top\right|_{ii} \frac{q_i}{q_i}\\
    &=\epsilon+\delta_i^2+\sum_j \left| C_1 F_1^{-1} C_1^\top \right|_{ij} \frac{q_j}{q_i}-\left|C_1F_1^{-1}C_1^\top\right|_{ii} \frac{q_j}{q_i}\\
    &>\sum_{j,i\neq j} \left| C_1 F_1^{-1} C_1^\top \right|_{ij} \frac{q_j}{q_i}\quad \forall i=1,\dots,c.
\end{align*}
Hence, $2\Gamma-\rho_\mathrm{P}^2\Gamma X\Gamma-C_1F_1^{-1}C_1^\top\succ 0$ according to Lemma~\ref{lem:diagonal_dominance}. Equality \eqref{eq:proof_max} implies the inequality
\begin{align*}
    2\Gamma - \rho_\mathrm{P}^2\Gamma X\Gamma - C_1 F_{1}^{-1} C_1^\top - (-\widehat{C}_2 +C_1 F_{1}^{-1}F_{12})\\ (F_2 -F_{12}^\top F_{1}^{-1} F_{12})^{-1} (-\widehat{C}_2^\top +F_{12}^\top F_{1}^{-1} C_1^\top) \succeq 0, 
\end{align*}
or,  equivalently, using $\Lambda=\Gamma^{-1}$, which is invertible as $\gamma_i\geq \epsilon$,
\begin{align*}
    2\Lambda - \rho_\mathrm{P}^2 X - \Lambda C_1 F_{1}^{-1} C_1^\top \Lambda - (-\Lambda \widehat{C}_2 + \Lambda C_1 F_{1}^{-1}F_{12})\\ (F_2 -F_{12}^\top F_{1}^{-1} F_{12})^{-1} (-\widehat{C}_2^\top \Lambda +F_{12}^\top F_{1}^{-1} C_1^\top \Lambda) \succeq 0, 
\end{align*}
which by Schur complements is equivalent to \eqref{eq:cert_conv+act+pool}, cmp. proof of Theorem~\ref{thm:conv}.
\end{pf}

\subsection{The last layer}
In the last layer, we directly set $X=Q=L_Q^\top L_Q$ instead of parameterizing some $X=L^\top L$ through $L$.

\begin{corollary}
    An affine fully connected layer \eqref{eq:fc} parameterized by
    \begin{align}\label{eq:parameterization_FNN_last}
        W = L_Q^{-1} V^\top L_{-}, ~  \begin{bmatrix} U\\ V \end{bmatrix} = \Cayley\left( \begin{bmatrix} Y\\ Z \end{bmatrix} \right)
    \end{align}
    where $L_Q=\chol(Q)$, $Y\in\bbR^{c\times c}$, $Z\in\bbR^{c_-\times c}$ satisfies \eqref{eq:cert_last_FC}.
\end{corollary}
\begin{pf}
    The proof follows along the lines of the proof in \cite[Theorem 5]{pauli2023lipschitz-syn}. We insert $V=L_-^{-\top}W_l^\top L_Q^\top$ into $U^\top U+V^\top V=I$ and obtain
    \begin{equation*}
        L_QW_lL_-^{-1} L_-^{-\top}W_l^\top L_Q^\top=I - U^\top U\preceq I
    \end{equation*}
    which by left/right multiplication with $L_Q^{-1}$/$L_Q^{-\top}$ yields
    \begin{equation*}
        W_l X_-^{-1}W_l^\top\preceq Q^{-1},
    \end{equation*}
    which in turn by two Schur complements implies \eqref{eq:cert_last_FC}.
\end{pf}

\subsection{LipKernel vs. Sandwich convolutional layers}\label{sec:comparison_Sandwich}
\begin{figure}
    \centering
    \includegraphics[width=0.45\textwidth]{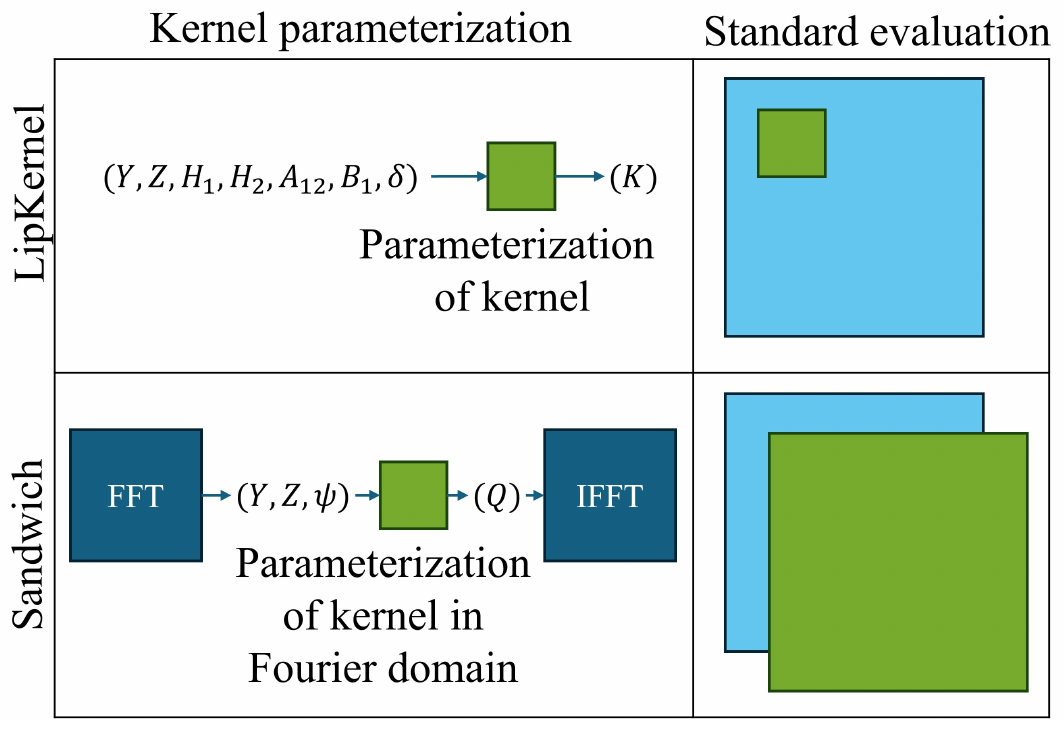}
    \caption{Differences between convolutional layers using LipKernel (ours) and Sandwich layers \citep{wang2023direct} in its parameterization complexity and its standard evaluation. The light blue boxes represent images and the green boxes the kernel.}
    \label{fig:differences_Sandwich_LipKernel}
\end{figure}

In this section, we have presented an LMI-based method for the parameterization of Lipschitz-bounded CNNs that we call LipKernel as we directly parameterize the kernel parameters. Similarly, the parameterization of Sandwich layers \citep{wang2023direct} is  based on LMIs, i.e., is also shows an increased expressivity over approaches using orthogonal layers and layers with constrained spectral norms, cmp. Subsection~\ref{sec:dissipative_layers}. In the following, we point out the differences between Sandwich and LipKernel convolutional layers, which are also illustrated in Fig.~\ref{fig:differences_Sandwich_LipKernel}. 

Both parameterizations Sandwich and LipKernel use the Cayley transform and require the computation of inverses at training time. However, LipKernel parameterizes the kernel parameters $K$ directly through the bijective mapping $(\bA,\bB,\bC,\bD)\mapsto K$ given by Lemma \ref{lem:min_real_2D}. This means that after training at inference time, we can construct $K$ from $(\bA,\bB,\bC,\bD)$ and then evaluate the trained CNN using this $K$. This is not possible using Sandwich layers \citep{wang2023direct}. At inference time Sandwich layers can either be evaluated using an full-image size kernel or in the Fourier domain, cmp. Fig.~\ref{fig:differences_Sandwich_LipKernel}. The latter requires the use of a fast Fourier transform and an inverse fast Fourier transform and the computation of inverses at inference time, making it computationally more costly than the evaluation of LipKernel layers.

We note that Sandwich requires circular padding instead of zero-padding and the implementation of \citet{wang2023direct} only takes input image sizes of the specific size of $2^n$, $n\in\bbN_0$. In this respect, LipKernel is more versatile than Sandwich, it can handle all kinds of zero-padding and accounts for pooling layers, which are not considered in \citep{wang2023direct}. 

\section{Numerical Experiments}\label{sec:experiments}

\subsection{Run-times for inference}\label{sec:runtimes}
First, we compare the run-times at inference, i.e., the time for evaluation of a fixed model after training, for varying numbers of channels, different input image sizes, and different kernel sizes for LipKernel, Sandwich, and Orthogon layers with randomly generated weights\footnote{The code is written in Python using Pytorch and was run on a standard i7 notebook. It is provided at \url{https://github.com/ppauli/2D-LipCNNs}.}.
\begin{itemize}
\item\textbf{Sandwich:} \citet{wang2023direct} suggest an LMI-based method using the Cayley transform, wherein convolutional layers are parameterized in the Fourier domain using circular padding, cmp. Subsection~\ref{sec:comparison_Sandwich}.
\item\textbf{Orthogon:} \citet{trockman2021orthogonalizing} use the Cayley transform to parameterize orthogonal layers. Convolutional layers are parameterized in the Fourier domain using circular padding.
\end{itemize}
The averaged run-times are shown in Fig.~\ref{fig:LipKernel_vs_Sandwich}. For all chosen channel, image, and kernel sizes the inference time of LipKernel is very short (from $<$1ms to around 100ms), whereas Sandwich layer and Orthogon layer evaluations are two to three orders of magnitude slower and increases significantly with channel and image sizes (from around 10ms to over 10s). Kernel size does not affect the run-time of either layer significantly.

A particular motivation of our work is to improve the robustness of NNs for use in real-time control systems. In this context, these inference-time differences can have a significant impact, both in terms of achievable sample rates (100Hz vs 0.1Hz) and latency in the feedback loop. Furthermore, it is increasingly the case that compute (especially NN inference) consumes a significant percentage of the power in mobile robots and other ``edge devices'' \citep{chen2020deep}. Significant reductions in inference time for robust NNs can therefore be a key enabler for use especially in  battery-powered systems.

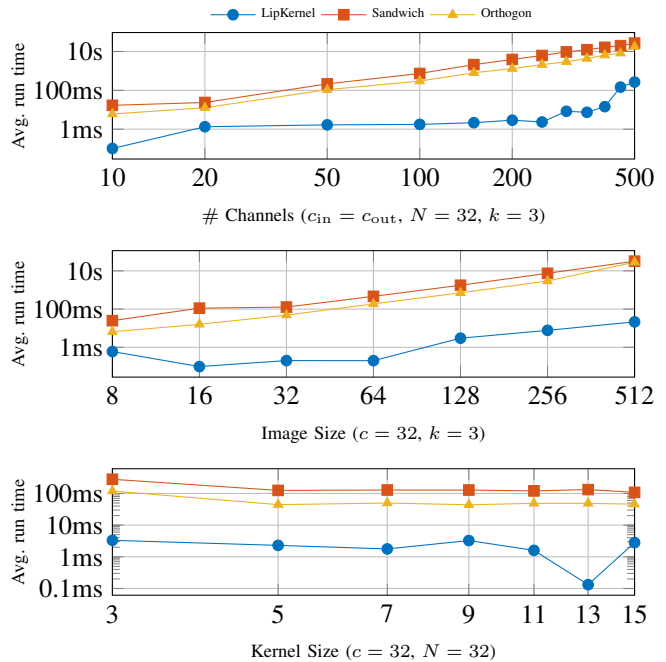
\begin{figure}
    \centering
    \input{figs/LipKernel_vs_Sandwich} 
    \caption{Inference times for LipKernel, Sandwich, and Orthogon layers with different numbers of channels $c=c_{\mathrm{in}}=c_{\mathrm{out}}$, input image sizes $N=N_1=N_2$, and kernel sizes $k=k_1=k_2$. For all layers, we have stride equal to $1$ and average the run-time over 10 different initializations.}
    \label{fig:LipKernel_vs_Sandwich}
\end{figure}

\subsection{Accuracy and robustness comparison}\label{sec:exp_comparison}
We next compare LipKernel to three other methods developed to train Lipschitz-bounded NNs in terms of accuracy and robustness. In particular, we compare LipKernel to Sandwich and Orthogon as well as vanilla and almost-orthogonal Lipschitz (AOL) NNs:
\begin{itemize}
\item\textbf{Vanilla:} Unconstrained neural network.
\item\textbf{AOL:} \citet{prach2022almost} introduce  a rescaling-based weight matrix parametrization to obtain AOL layers which are 1-Lipschitz. Like LipKernel layers, at inference, convolutional AOL layers can be evaluated in standard form.
\end{itemize}

We train classifying CNNs on the MNIST dataset \citep{mnist} of size $32\times 32$ images with CNN architectures 2C2F: $c(16,4,2).c(32,4,2).f(100).f(10)$, 2CP2F: $c(16,4,1).p(\mathrm{av},2,2).c(32,4,1).p(\mathrm{av},2,2).f(100)$ $.f(10)$, wherein by $c(C,K,S)$, we denote a convolutional layer with $C$ output channels, kernel size $K$, and stride $S$, by $f(N)$ a fully connected layer with $N$ output neurons, and by $p(\text{type},K,S)$ an `$\mathrm{av}$' or `$\mathrm{max}$' pooling layer.

In Table \ref{table:main_results}, we show the clean accuracy, i.e., the test accuracy on unperturbed test data, the certified robust accuracy, and the robustness under the $\ell_2$ projected gradient descent (PGD) adversarial attack of the trained NNs. The certified robust accuracy is a robustness metric for NNs that gives the fraction of test data points that are guaranteed to remain correct under all perturbations from an $\epsilon$-ball. It is obtained by identifying all test data points $x$ with classification margin $\calM_f(x)$ greater than $\sqrt{2}\rho\epsilon$, where $\rho$ is the NN's upper bound on the Lipschitz constant \citep{tsuzuku2018lipschitz}. The $\ell_2$ PGD attack is a white box multi-step attack that modifies each input data point by maximizing the loss within an $\ell_2$ $\epsilon$-ball around that point \citep{madry2017towards}. The accuracy under $\ell_2$ PGD attacks gives the fraction of attacked test data points which are correctly classified.

First, we note that LipKernel is general and flexible in the sense that we can use it in both the 2C2F and the 2CP2F architectures, whereas Sandwich and Orthogon are limited to image sizes of $2^n$ and to circular padding and AOL does not support strided convolutions. Comparing LipKernel to Orthogon and AOL, we notice better expressivity in the higher clean accuracy and significantly better robustness with the stronger Lipschitz bounds of 1 and 2. In comparison to Sandwich, LipKernel achieves comparable but slightly lower expressivity and robustness. However as discussed above it is more flexible in terms of architecture and has a significant advantage in terms of inference times.

In Figure \ref{fig:robustness_accuracy_tradeoff_2D}, we plot the achieved clean test accuracy over the Lipschitz lower bound for 2C2F and 2CP2F for LipKernel and the other methods, clearly recognizing the trade-off between accuracy and robustness. Again, we see that LipKernel shows better expressivity than Orthogon and AOL and similar performance to Sandwich.

\begin{table*}[t]
  \centering
  \caption{Empirical lower Lipschitz bounds, clean accuracy, certified robust accuracy and adversarial robustness under $\ell_2$ PGD attack for vanilla, AOL, Orthogon, Sandwich, and LipKernel NNs using the architectures 2C2F and 2CP2F with ReLU activations, each trained for 20 epochs and averaged for three different initializations.
}
  \renewcommand{\arraystretch}{0.9}
    \resizebox{\textwidth}{!}{%
    \begin{tabular}{lcrrrrrrrrrr}
    \toprule
    \multirow{2}{*}{\textbf{Model}} & \multirow{2}{*}{\textbf{Method}} & \multirow{2}{*}{\textbf{Cert. UB}} & \multirow{2}{*}{\textbf{Emp. LB}} & \multirow{2}{*}{\textbf{Test acc.}} & \multicolumn{3}{c}{\textbf{Cert. robust acc.}}  & \multicolumn{3}{c}{\textbf{$\ell_2$ PGD Adv. test acc.}} \\ \cmidrule{6-11}
    &&&&&$\epsilon=\frac{36}{255}$ & $\epsilon=\frac{72}{255}$ & $\epsilon=\frac{108}{255}$ &$\epsilon=1.0$ & $\epsilon=2.0$ & $\epsilon=3.0$ \\
    \midrule 
     \multirow{10}{*}{2C2F} 
     & Vanilla  & -- & 221.7 &	99.0\% &	0.0\%  & 0.0\%  & 0.0\% & 69.5\% &	61.9\% & 59.6\% \\ \cmidrule{2-11}
     & Orthogon & 1  & 0.960 &	94.6\% &	92.9\% & 91.0\% & 88.3\% & 83.7\% &	65.2\% & 60.4\% \\
     & Sandwich & 1  & 0.914	&   97.3\% &    96.3\% & 95.2\% & 93.8\% & 90.5\% &76.5\% & 72.0\% \\
     & LipKernel  & 1  & 0.952 &	96.6\% &	95.6\% & 94.3\% & 92.6\% & 88.3\% &	72.2\% & 67.8\% \\ \cmidrule{2-11}
     & Orthogon & 2  & 1.744	&   97.7\% &	96.3\% & 94.4\% & 91.8\% & 89.1\% &	66.0\% & 58.2\% \\
     & Sandwich & 2  & 1.703	&   98.9\% &	98.2\% & 97.0\% & 95.4\% & 93.1\% &	74.0\% & 67.2\% \\
     & LipKernel  & 2  & 1.703 & 98.2\% &    97.1\% &    95.6\% & 93.6\% & 89.8\% &    66.1\% & 58.9\% \\ \cmidrule{2-11}
     & Orthogon & 4  & 2.894 &	98.8\% &	97.4\% & 94.4\% & 88.6\% & 89.6\% &	56.0\% & 46.0\% \\
     & Sandwich & 4  & 2.969	&   99.3\% &    98.4\% & 96.9\% & 93.6\% & 92.5\% &   63.3\%  & 54.0\% \\
     & LipKernel  & 4  & 3.110 &	98.9\% &	97.5\% &  95.3\% & 91.3\% &  88.6\% & 49.6\% & 39.7\% \\  \midrule 
     \multirow{7}{*}{2CP2F} 
     & Vanilla & -- & 148.0 &   99.3\%  &     0.0\% &     0.0\% & 0.0\% &73.2\% &56.2\% & 53.7\% \\ \cmidrule{2-11}   
     & AOL      & 1  & 0.926 & 	88.7\% &    85.5\% &    81.7\% & 77.2\% &	70.6\% &	49.2\% & 44.6\% \\
     & LipKernel  & 1  & 0.759	&   91.7\% & 88.0\% & 83.1\% & 77.3\% &	77.3\% &	57.2\% & 52.2\% \\ \cmidrule{2-11}
     & AOL      & 2  & 1.718 &	93.0\% &	89.9\% &    85.9\% & 80.4\% & 75.8\% &	46.6\% & 38.1\% \\ 
     & LipKernel  & 2  & 1.312	&   94.9\% & 91.1\% & 85.4\% & 77.8\% &	80.9\% &53.8\% & 45.8\% \\ \cmidrule{2-11}
     & AOL      & 4  & 2.939 & 	95.9\% &	92.4\% & 86.2\% & 76.3\% & 78.2\% &	37.0\% & 29.6\% \\
     & LipKernel  & 4  & 2.455	&   97.1\% & 93.7\% & 87.2\% & 75.7\% &	80.0\% &	36.8\% & 29.0\% \\
    \bottomrule 
    \end{tabular}%
    }    
  \label{table:main_results}%
\end{table*}%

\begin{figure}
    \centering
    \input{figs/robustness_accuracy_tradeoff_2D} 
    \caption{Robustness accuracy trade-off for 2C2F (left) 2CP2F (right) for NNs averaged over three initializations.}
    \label{fig:robustness_accuracy_tradeoff_2D}
\end{figure}
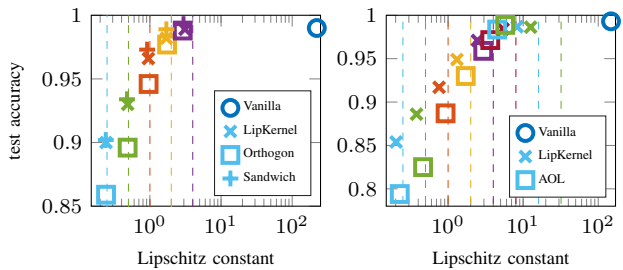

\section{Conclusion}
We have introduced LipKernel, an expressive and versatile parameterization for Lipschitz-bounded CNNs. Our parameterization of convolutional layers is based on a 2-D state space representation of the Roesser type that, unlike  parameterizations in the Fourier domain, allows to directly parameterize the kernel parameters of convolutional layers. This in turn enables fast evaluation at inference time making LipKernel  especially useful for real-time control systems. Our parameterization satisfies layer-wise LMI constraints that render the individual layers incrementally dissipative and the end-to-end mapping Lipschitz-bounded. Furthermore, our general framework can incorporate any dissipative layer. 


\renewcommand*{\bibfont}{\fontsize{9}{9.1}\selectfont}
\bibliographystyle{ifacconf-harvard}        
\bibliography{autosam}           




\bgroup
\setlength{\columnsep}{4pt}
\small
\begin{wrapfigure}[10]{l}{1in}
	\includegraphics[width=1in,height=1.25in,clip,keepaspectratio]{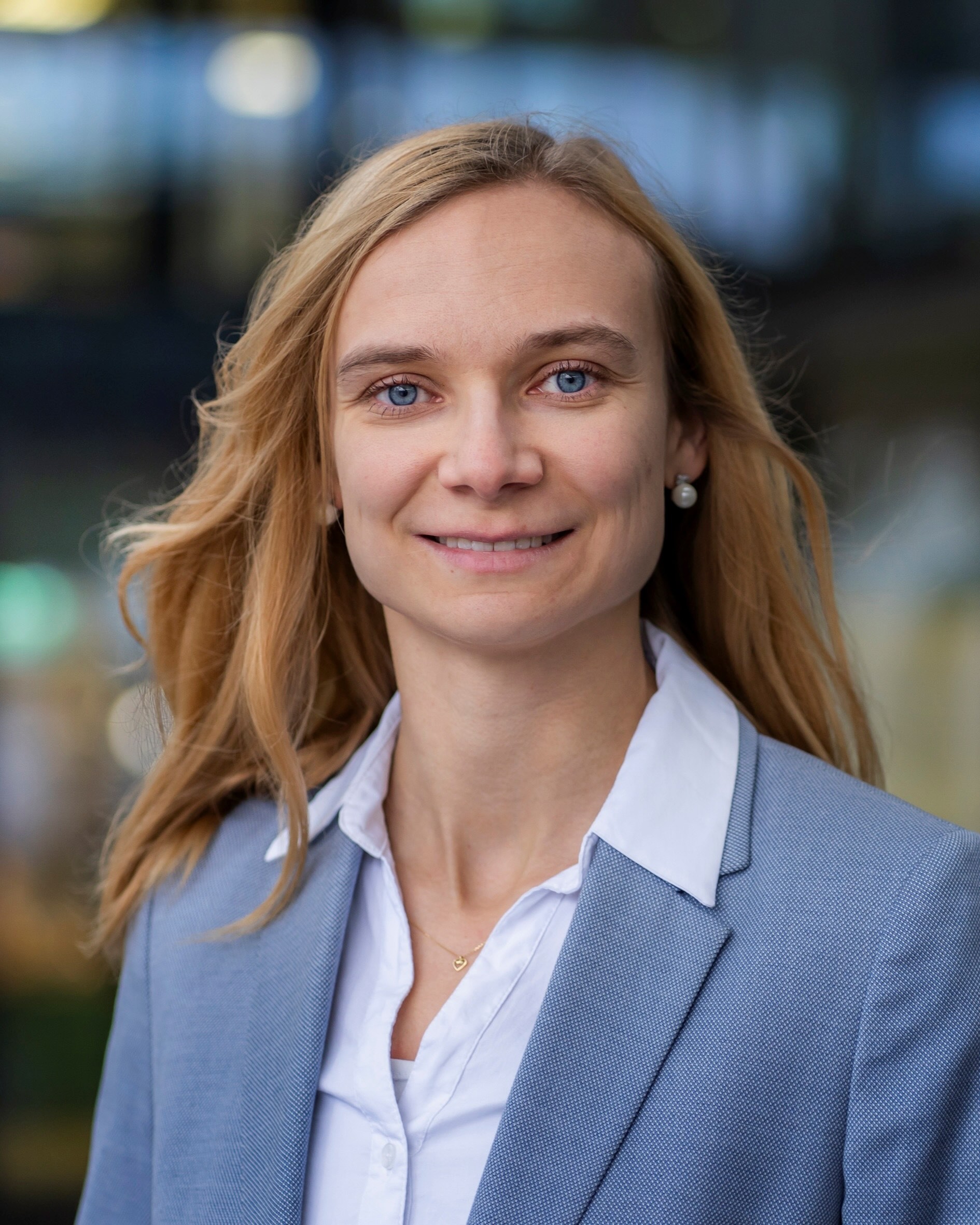}
	\vspace{-12pt}
\end{wrapfigure}
\noindent
\textbf{Patricia Pauli} received master's degrees in mechanical engineering and computational engineering from the Technical University of Darmstadt, Germany, in 2019. She received a PhD from the University of Stuttgart, Germany, in 2025. Since 2025, she has been an Assistant Professor in the Department of Mechanical Engineering at Eindhoven University of Technology. Her research interests are in robust machine learning and learning-based control.

\begin{wrapfigure}[10]{l}{1in}
	\vspace{-8pt}
	\includegraphics[width=1in,height=1.25in, clip, keepaspectratio]{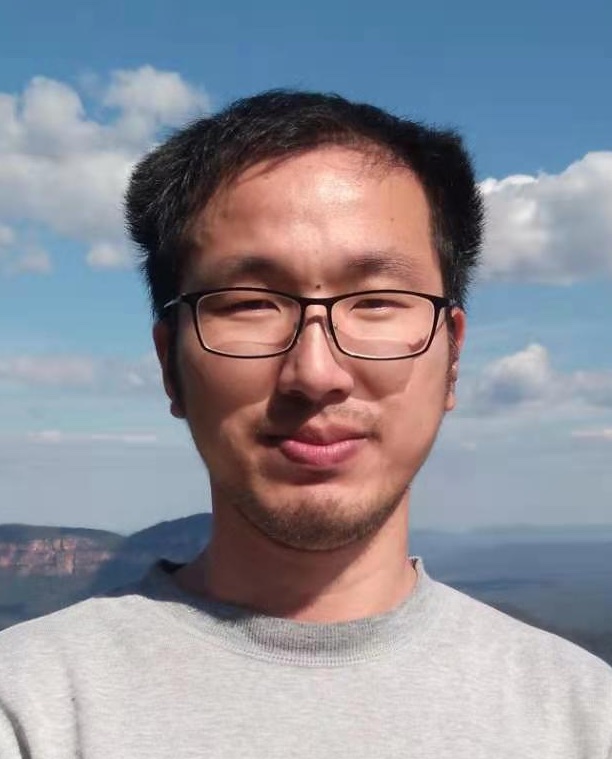}
	\vspace{-12pt}
\end{wrapfigure}
\noindent
\textbf{Ruigang Wang} received the Ph.D. degree in chemical engineering from The University of New South Wales (UNSW), Sydney, NSW, Australia, in 2017. From 2017 to 2018, he worked as a Postdoctoral Fellow with the UNSW. He is currently a Postdoctoral Fellow with the Australian Centre for Robotics, The University of Sydney, Sydney. His research interests include contraction-based control, estimation, and learning for nonlinear systems.

\begin{wrapfigure}[10]{l}{1in}
	\vspace{-8pt}
	\includegraphics[width=1in,height=1.25in,clip,keepaspectratio]{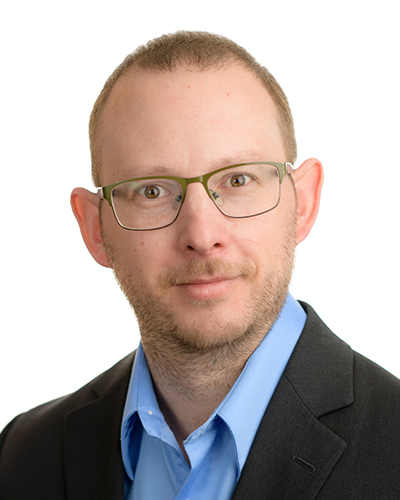}
	\vspace{-12pt}
\end{wrapfigure}
\noindent
\textbf{Ian R. Manchester} received the B.E. (Hons 1) and Ph.D. degrees in Electrical Engineering from the University of New South Wales, Australia, in 2002 and 2006, respectively. He was a Researcher with Umeå University, Umeå, Sweden, and the Massachusetts Institute of Technology, Cambridge, MA, USA. In 2012, he joined the Faculty with the University of Sydney, Camperdown, NSW, Australia, where he is currently a Professor of mechatronic engineering, the Director of the Australian Centre for Robotics (ACFR), and Director of the Australian Robotic Inspection and Asset Management Hub (ARIAM). His research interests include optimization and learning methods for nonlinear system analysis, identification, and control, and the applications in robotics and biomedical engineering.

\begin{wrapfigure}[10]{l}{1in}
	\vspace{-8pt}
	\includegraphics[width=1in,height=1.3in,clip,keepaspectratio]{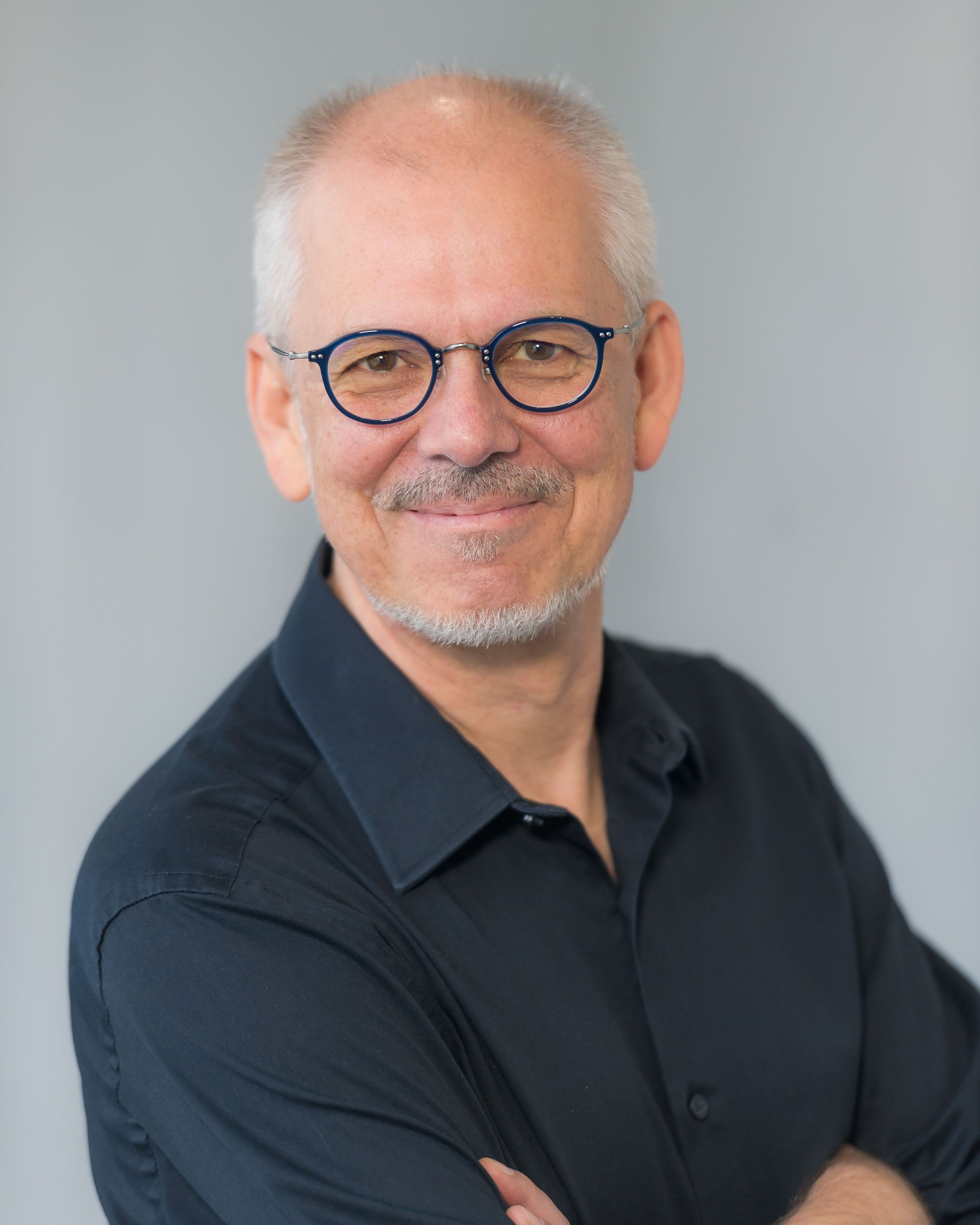}
	\vspace{-12pt}
\end{wrapfigure}
\noindent
\textbf{Frank Allg\"ower} studied engineering cybernetics and applied mathematics in Stuttgart and with the University of California, Los Angeles (UCLA), CA, USA, respectively, and received the Ph.D. degree from the University of Stuttgart, Stuttgart, Germany. Since 1999, he has been the Director of the Institute for Systems Theory and Automatic Control and a professor with the University of Stuttgart. His research interests include predictive control, data-based control, networked control, cooperative control, and nonlinear control with application to a wide range of fields including systems biology. Dr. Allgöwer was the President of the International Federation of Automatic Control (IFAC) in 2017–2020 and the Vice President of the German Research Foundation DFG in 2012–2020.

\egroup

\end{document}

%% file: figs/motivational_example.tex
%
%
\definecolor{mycolor1}{rgb}{0.00000,0.44700,0.74100}%
\definecolor{mycolor2}{rgb}{0.85000,0.32500,0.09800}%
\definecolor{mycolor3}{rgb}{0.92900,0.69400,0.12500}%
\begin{tikzpicture}

\begin{axis}[%
width=2.3in,
height=1.1in,
at={(0.542in,0.441in)},
scale only axis,
xmin=-1.5707963267949,
xmax=1.5707963267949,
xtick={-1.5707963267949,0,1.5707963267949},
xticklabels={{$\text{-}\pi\text{/2}$},{0},{$\pi\text{/2}$}},
xlabel style={font=\color{white!15!black}},
xlabel={$u_1$},
ymin=0,
ymax=1.1,
ylabel style={font=\color{white!15!black}},
ylabel={$y_l$},
axis background/.style={fill=white},
xmajorgrids,
ymajorgrids,
legend style={at={(0.5,0.03)}, anchor=south, legend cell align=left, align=left, draw=white!15!black, font=\scriptsize}
]
\addplot [color=mycolor1, line width=1.0pt,dotted]
  table[row sep=crcr]{%
-1.5707963267949	6.12323399573677e-17\\
-1.53906306766773	0.0317279334980676\\
-1.50732980854056	0.0634239196565646\\
-1.47559654941339	0.0950560433041826\\
-1.44386329028622	0.126592453573749\\
-1.41213003115905	0.15800139597335\\
-1.38039677203188	0.18925124436041\\
-1.34866351290471	0.220310532786541\\
-1.31693025377754	0.251147987181079\\
-1.28519699465037	0.28173255684143\\
-1.2534637355232		0.312033445698487\\
-1.22173047639603	0.342020143325669\\
-1.18999721726886	0.371662455660328\\
-1.15826395814169	0.400930535406614\\
-1.12653069901452	0.429794912089172\\
-1.09479743988735	0.45822652172741\\
-1.06306418076018	0.486196736100469\\
-1.03133092163301	0.513677391573406\\
-0.999597662505843	0.540640817455598\\
-0.967864403378674	0.567059863862771\\
-0.936131144251504	0.592907929054641\\
-0.904397885124334	0.618158986220605\\
-0.872664625997165	0.642787609686539\\
-0.840931366869995	0.666769000516292\\
-0.809198107742826	0.690079011482112\\
-0.777464848615656	0.712694171378863\\
-0.745731589488486	0.734591708657533\\
-0.713998330361317	0.755749574354258\\
-0.682265071234147	0.776146464291757\\
-0.650531812106977	0.795761840530832\\
-0.618798552979808	0.814575952050336\\
-0.587065293852638	0.832569854634771\\
-0.555332034725469	0.849725429949514\\
-0.523598775598299	0.866025403784439\\
-0.491865516471129	0.881453363447582\\
-0.46013225734396	0.895993774291336\\
-0.42839899821679	0.909631995354518\\
-0.39666573908962	0.922354294104581\\
-0.364932479962451	0.934147860265107\\
-0.333199220835281	0.945000818714668\\
-0.301465961708111	0.954902241444074\\
-0.269732702580942	0.963842158559942\\
-0.237999443453772	0.971811568323542\\
-0.206266184326603	0.978802446214779\\
-0.174532925199433	0.984807753012208\\
-0.142799666072263	0.989821441880933\\
-0.111066406945094	0.993838464461254\\
-0.0793331478179241	0.996854775951942\\
-0.0475998886907544	0.998867339183008\\
-0.0158666295635848	0.999874127673875\\
0.0158666295635848	0.999874127673875\\
0.0475998886907544	0.998867339183008\\
0.0793331478179241	0.996854775951942\\
0.111066406945094	0.993838464461254\\
0.142799666072263	0.989821441880933\\
0.174532925199433	0.984807753012208\\
0.206266184326603	0.978802446214779\\
0.237999443453772	0.971811568323542\\
0.269732702580942	0.963842158559942\\
0.301465961708111	0.954902241444074\\
0.333199220835281	0.945000818714668\\
0.364932479962451	0.934147860265107\\
0.39666573908962	0.922354294104581\\
0.42839899821679	0.909631995354518\\
0.46013225734396	0.895993774291336\\
0.491865516471129	0.881453363447582\\
0.523598775598299	0.866025403784439\\
0.555332034725469	0.849725429949514\\
0.587065293852638	0.832569854634771\\
0.618798552979808	0.814575952050336\\
0.650531812106977	0.795761840530832\\
0.682265071234147	0.776146464291757\\
0.713998330361317	0.755749574354258\\
0.745731589488486	0.734591708657533\\
0.777464848615656	0.712694171378863\\
0.809198107742826	0.690079011482112\\
0.840931366869995	0.666769000516292\\
0.872664625997165	0.642787609686539\\
0.904397885124334	0.618158986220605\\
0.936131144251504	0.592907929054641\\
0.967864403378674	0.567059863862771\\
0.999597662505843	0.540640817455598\\
1.03133092163301	0.513677391573406\\
1.06306418076018	0.486196736100469\\
1.09479743988735	0.45822652172741\\
1.12653069901452	0.429794912089172\\
1.15826395814169	0.400930535406614\\
1.18999721726886	0.371662455660328\\
1.22173047639603	0.342020143325669\\
1.2534637355232	0.312033445698487\\
1.28519699465037	0.28173255684143\\
1.31693025377754	0.251147987181079\\
1.34866351290471	0.220310532786541\\
1.38039677203188	0.18925124436041\\
1.41213003115905	0.15800139597335\\
1.44386329028622	0.126592453573749\\
1.47559654941339	0.0950560433041826\\
1.50732980854056	0.0634239196565646\\
1.53906306766773	0.0317279334980676\\
1.5707963267949	6.12323399573677e-17\\
};
\addlegendentry{Cosine}

\addplot [color=mycolor2, line width=1.0pt, solid]
  table[row sep=crcr]{%
-1.5707963267949	-0.0275725908696939\\
-1.53906306766773	-0.00466438250242862\\
-1.50732980854056	0.0189457049685666\\
-1.47559654941339	0.0432393683256316\\
-1.44386329028622	0.0681941344451933\\
-1.41213003115905	0.093783255084379\\
-1.38039677203188	0.119975641544006\\
-1.34866351290471	0.146735843269196\\
-1.31693025377754	0.174024073832051\\
-1.28519699465037	0.20179628696679\\
-1.2534637355232	0.230004304408291\\
-1.22173047639603	0.258595996239151\\
-1.18999721726886	0.287515513305261\\
-1.15826395814169	0.316703570049175\\
-1.12653069901452	0.346097774873925\\
-1.09479743988735	0.375633003931233\\
-1.06306418076018	0.405241813073377\\
-1.03133092163301	0.434854881663126\\
-0.999597662505843	0.464401481044458\\
-0.967864403378674	0.493809959776498\\
-0.936131144251504	0.523008237255053\\
-0.904397885124334	0.551924297112098\\
-0.872664625997165	0.580486671804683\\
-0.840931366869995	0.608624910081058\\
-0.809198107742826	0.636270019532098\\
-0.777464848615656	0.663354877178573\\
-0.745731589488486	0.689814601978697\\
-0.713998330361317	0.715586884227375\\
-0.682265071234147	0.740612268015566\\
-0.650531812106977	0.764834384179214\\
-0.618798552979808	0.788200132446501\\
-0.587065293852638	0.810659812745697\\
-0.555332034725469	0.832167206824233\\
-0.523598775598299	0.852679612418723\\
-0.491865516471129	0.872157833178587\\
-0.46013225734396	0.890566128363017\\
-0.42839899821679	0.907872126990274\\
-0.39666573908962	0.924046711614754\\
-0.364932479962451	0.939063877242604\\
-0.333199220835281	0.952900571077977\\
-0.301465961708111	0.965536518830732\\
-0.269732702580942	0.976954043227213\\
-0.237999443453772	0.987137880165061\\
-0.206266184326603	0.996074997658307\\
-0.174532925199433	1.00375442234729\\
-0.142799666072263	1.01016707791544\\
-0.111066406945094	1.01530563927633\\
-0.0793331478179241	1.01916440588174\\
-0.0475998886907544	1.02173919696595\\
-0.0158666295635848	1.02302727099009\\
0.0158666295635848	1.02302727099009\\
0.0475998886907544	1.02173919696595\\
0.0793331478179241	1.01916440588174\\
0.111066406945094	1.01530563927633\\
0.142799666072263	1.01016707791544\\
0.174532925199433	1.00375442234729\\
0.206266184326603	0.996074997658307\\
0.237999443453772	0.987137880165061\\
0.269732702580942	0.976954043227213\\
0.301465961708111	0.965536518830732\\
0.333199220835281	0.952900571077977\\
0.364932479962451	0.939063877242604\\
0.39666573908962	0.924046711614754\\
0.42839899821679	0.907872126990274\\
0.46013225734396	0.890566128363017\\
0.491865516471129	0.872157833178587\\
0.523598775598299	0.852679612418723\\
0.555332034725469	0.832167206824233\\
0.587065293852638	0.810659812745697\\
0.618798552979808	0.788200132446501\\
0.650531812106977	0.764834384179214\\
0.682265071234147	0.740612268015566\\
0.713998330361317	0.715586884227375\\
0.745731589488486	0.689814601978697\\
0.777464848615656	0.663354877178573\\
0.809198107742826	0.636270019532098\\
0.840931366869995	0.608624910081058\\
0.872664625997165	0.580486671804683\\
0.904397885124334	0.551924297112098\\
0.936131144251504	0.523008237255053\\
0.967864403378674	0.493809959776498\\
0.999597662505843	0.464401481044458\\
1.03133092163301	0.434854881663126\\
1.06306418076018	0.405241813073377\\
1.09479743988735	0.375633003931233\\
1.12653069901452	0.346097774873925\\
1.15826395814169	0.316703570049175\\
1.18999721726886	0.287515513305261\\
1.22173047639603	0.258595996239151\\
1.2534637355232	0.230004304408291\\
1.28519699465037	0.20179628696679\\
1.31693025377754	0.174024073832051\\
1.34866351290471	0.146735843269196\\
1.38039677203188	0.119975641544006\\
1.41213003115905	0.093783255084379\\
1.44386329028622	0.0681941344451933\\
1.47559654941339	0.0432393683256316\\
1.50732980854056	0.0189457049685666\\
1.53906306766773	-0.00466438250242862\\
1.5707963267949	-0.0275725908696939\\
};
\addlegendentry{Dissipative layers}

\addplot [color=mycolor3, line width=1.0pt,dashed]
  table[row sep=crcr]{%
-1.5707963267949	0.40634618857065\\
-1.53906306766773	0.420328423634522\\
-1.50732980854056	0.434090150346429\\
-1.47559654941339	0.447620776915736\\
-1.44386329028622	0.460910384262217\\
-1.41213003115905	0.473949736735238\\
-1.38039677203188	0.486730288244439\\
-1.34866351290471	0.499244183945017\\
-1.31693025377754	0.511484257659181\\
-1.28519699465037	0.52344402524914\\
-1.2534637355232	0.535117674185811\\
-1.22173047639603	0.546500049581194\\
-1.18999721726886	0.557586636971056\\
-1.15826395814169	0.568373542148338\\
-1.12653069901452	0.578857468356584\\
-1.09479743988735	0.589035691157189\\
-1.06306418076018	0.598906031284413\\
-1.03133092163301	0.608466825798553\\
-0.999597662505843	0.617716897840619\\
-0.967864403378674	0.626655525281852\\
-0.936131144251504	0.635282408548881\\
-0.904397885124334	0.643597637890662\\
-0.872664625997165	0.651601660337018\\
-0.840931366869995	0.659295246581088\\
-0.809198107742826	0.666679457999529\\
-0.777464848615656	0.673755614005497\\
-0.745731589488486	0.680525259910327\\
-0.713998330361317	0.686990135450992\\
-0.682265071234147	0.6931521441219\\
-0.650531812106977	0.699013323431741\\
-0.618798552979808	0.704575816189066\\
-0.587065293852638	0.709841842904171\\
-0.555332034725469	0.714813675379913\\
-0.523598775598299	0.719493611550238\\
-0.491865516471129	0.723883951612643\\
-0.46013225734396	0.727986975489464\\
-0.42839899821679	0.731804921642862\\
-0.39666573908962	0.735339967259598\\
-0.364932479962451	0.73859420981418\\
-0.333199220835281	0.74156965001263\\
-0.301465961708111	0.744268176113963\\
-0.269732702580942	0.746691549622374\\
-0.237999443453772	0.748841392340081\\
-0.206266184326603	0.750719174768647\\
-0.174532925199433	0.752326205845332\\
-0.142799666072263	0.75366362400059\\
-0.111066406945094	0.754732389522993\\
-0.0793331478179241	0.755533278218734\\
-0.0475998886907544	0.756066876354198\\
-0.0158666295635848	0.756333576871848\\
0.0158666295635848	0.756333576871848\\
0.0475998886907544	0.756066876354198\\
0.0793331478179241	0.755533278218734\\
0.111066406945094	0.754732389522993\\
0.142799666072263	0.75366362400059\\
0.174532925199433	0.752326205845332\\
0.206266184326603	0.750719174768647\\
0.237999443453772	0.748841392340081\\
0.269732702580942	0.746691549622374\\
0.301465961708111	0.744268176113963\\
0.333199220835281	0.74156965001263\\
0.364932479962451	0.73859420981418\\
0.39666573908962	0.735339967259598\\
0.42839899821679	0.731804921642862\\
0.46013225734396	0.727986975489464\\
0.491865516471129	0.723883951612643\\
0.523598775598299	0.719493611550238\\
0.555332034725469	0.714813675379913\\
0.587065293852638	0.709841842904171\\
0.618798552979808	0.704575816189066\\
0.650531812106977	0.699013323431741\\
0.682265071234147	0.6931521441219\\
0.713998330361317	0.686990135450992\\
0.745731589488486	0.680525259910327\\
0.777464848615656	0.673755614005497\\
0.809198107742826	0.666679457999529\\
0.840931366869995	0.659295246581088\\
0.872664625997165	0.651601660337018\\
0.904397885124334	0.643597637890662\\
0.936131144251504	0.635282408548881\\
0.967864403378674	0.626655525281852\\
0.999597662505843	0.617716897840619\\
1.03133092163301	0.608466825798553\\
1.06306418076018	0.598906031284413\\
1.09479743988735	0.589035691157189\\
1.12653069901452	0.578857468356584\\
1.15826395814169	0.568373542148338\\
1.18999721726886	0.557586636971056\\
1.22173047639603	0.546500049581194\\
1.2534637355232	0.535117674185811\\
1.28519699465037	0.52344402524914\\
1.31693025377754	0.511484257659181\\
1.34866351290471	0.499244183945017\\
1.38039677203188	0.486730288244439\\
1.41213003115905	0.473949736735238\\
1.44386329028622	0.460910384262217\\
1.47559654941339	0.447620776915736\\
1.50732980854056	0.434090150346429\\
1.53906306766773	0.420328423634522\\
1.5707963267949	0.40634618857065\\
};
\addlegendentry{1-Lipschitz layers}

\end{axis}
\end{tikzpicture}%

%% file: figs/LipKernel_vs_Sandwich.tex
\definecolor{mycolor1}{rgb}{0.00000,0.44700,0.74100}%
\definecolor{mycolor2}{rgb}{0.85000,0.32500,0.09800}%
\definecolor{mycolor3}{rgb}{0.92900,0.69400,0.12500}%
\definecolor{mycolor4}{rgb}{0.49400,0.18400,0.55600}%
\definecolor{mycolor5}{rgb}{0.46600,0.67400,0.18800}%
\definecolor{mycolor6}{rgb}{0.30100,0.74500,0.93300}%
\definecolor{mycolor7}{rgb}{0.63500,0.07800,0.18400}%
    
    \begin{tikzpicture}
        \begin{axis}[
            xlabel={\scriptsize{$\#$ Channels ($c_{\mathrm{in}}=c_{\mathrm{out}}$, $N=32$, $k=3$)}},
            ylabel={\scriptsize{Avg. run time}},
            grid=both,
            legend pos=north west,
            width=0.47\textwidth,
            height=0.18\textwidth,
            xmin=10,
            xmax=500,
            xmode=log,
            ymode=log,
            xtick = {10,20,50,100,200,500},
            xticklabels = {10,20,50,100,200,500},
            ytick = {1,100,10000},
            yticklabels = {1ms,100ms,10s},
            legend style={anchor=north, at = {(0.5,1.28)}, legend columns=3, draw=none}
        ]
        \addplot[color = mycolor1, mark = *] table[x index=0, y index=1, col sep=comma] {Data_LipKernel_vs_Sandwich/lipkernel_times_channels.csv};
        \addlegendentry{\tiny{LipKernel}}
        \addplot[color = mycolor2, mark = square*] table[x index=0, y index=1, col sep=comma] {Data_LipKernel_vs_Sandwich/sandwich_times_channels.csv};
        \addlegendentry{\tiny{Sandwich}}
        \addplot[color = mycolor3, mark = triangle*] table[x index=0, y index=1, col sep=comma] {Data_LipKernel_vs_Sandwich/Orthogon_times_channels.csv};
        \addlegendentry{\tiny{Orthogon}}
        \end{axis}
        \begin{axis}[
            xlabel={\scriptsize{Image Size ($c=32$, $k=3$)}},
            ylabel={\scriptsize{Avg. run time}},
            grid=both,
            legend pos=north west,
            width=0.47\textwidth,
            height=0.18\textwidth,
            xmin=8,
            xmax=512,
            xmode=log,
            ymode=log,
            xtick = {8,16,32,64,128,256,512},
            xticklabels = {8,16,32,64,128,256,512},
            ytick = {1,100,10000},
            yticklabels = {1ms,100ms,10s},
            at={(0in,-0.16\textwidth)},
        ]
        \addplot[color = mycolor1, mark = *] table[x index=0, y index=1, col sep=comma] {Data_LipKernel_vs_Sandwich/lipkernel_times_spatial.csv};
        \addplot[color = mycolor2, mark = square*] table[x index=0, y index=1, col sep=comma] {Data_LipKernel_vs_Sandwich/sandwich_times_spatial.csv};
        \addplot[color = mycolor3, mark = triangle*] table[x index=0, y index=1, col sep=comma] {Data_LipKernel_vs_Sandwich/Orthogon_times_spatial.csv};
        \end{axis}
        \begin{axis}[
            xlabel={\scriptsize{Kernel Size ($c=32$, $N=32$)}},
            ylabel={\scriptsize{Avg. run time}},
            grid=major,
            legend pos=north west,
            width=0.47\textwidth,
            height=0.18\textwidth,
            xmin=3,
            xmax=15,
            xmode=log,
            ymode=log,
            xtick = {3,5,7,9,11,13,15},
            xticklabels = {3,5,7,9,11,13,15},
            ytick = {0.1,1,10,100},
            yticklabels = {0.1ms,1ms,10ms,100ms},
            at={(0in,-0.32\textwidth)},
        ]
        \addplot[color = mycolor1, mark = *] table[x index=0, y index=1, col sep=comma] {Data_LipKernel_vs_Sandwich/lipkernel_times_kernel.csv};
        \addplot[color = mycolor2, mark = square*] table[x index=0, y index=1, col sep=comma] {Data_LipKernel_vs_Sandwich/sandwich_times_kernel.csv};
        \addplot[color = mycolor3, mark = triangle*] table[x index=0, y index=1, col sep=comma] {Data_LipKernel_vs_Sandwich/Orthogon_times_kernel.csv};
        \end{axis}
    \end{tikzpicture}

%% file: figs/robustness_accuracy_tradeoff_2D.tex
%
%
\definecolor{mycolor1}{rgb}{0.00000,0.44700,0.74100}%
\definecolor{mycolor2}{rgb}{0.85000,0.32500,0.09800}%
\definecolor{mycolor3}{rgb}{0.92900,0.69400,0.12500}%
\definecolor{mycolor4}{rgb}{0.49400,0.18400,0.55600}%
\definecolor{mycolor5}{rgb}{0.46600,0.67400,0.18800}%
\definecolor{mycolor6}{rgb}{0.30100,0.74500,0.93300}%
\definecolor{mycolor7}{rgb}{0.63500,0.07800,0.18400}%
\begin{tikzpicture}

\begin{axis}[%
width=1.2in,
height=1in,
at={(0.406in,0.244in)},
scale only axis,
xmode=log,
xmin=0.15,
xmax=250,
xminorticks=true,
ymin=0.85,
ymax=1,
xlabel={\scriptsize{Lipschitz constant}},
xlabel near ticks,
ylabel={\scriptsize{test accuracy}},
ylabel near ticks,
axis background/.style={fill=white},
legend style={at={(0.96,0.33)}, anchor=east, legend cell align=left, align=left, draw=white!15!black},
every tick label/.append style={font=\footnotesize}
]
\addplot [color=mycolor1, line width=1.5pt, only marks, mark size=3.0pt, mark=o, mark options={solid, mycolor1}]
  table[row sep=crcr]{%
221.65	0.99\\
};
\addlegendentry{\tiny{Vanilla}}

\addplot [color=mycolor5, line width=1.5pt, only marks, mark size=3.0pt, mark=x, mark options={solid, mycolor6}]
  table[row sep=crcr]{%
  0.242	0.900\\
};
\addlegendentry{\tiny{LipKernel}}

\addplot [color=mycolor6, line width=1.5pt, only marks, mark size=3.0pt, mark=square, mark options={solid, mycolor6}]
  table[row sep=crcr]{%
0.242 0.859\\
};
\addlegendentry{\tiny{Orthogon}}

\addplot [color=mycolor6, line width=1.5pt, only marks, mark size=3.0pt, mark=+, mark options={solid, mycolor6}]
  table[row sep=crcr]{%
0.240 0.902\\
};
\addlegendentry{\tiny{Sandwich}}

\addplot [color=mycolor5, line width=1.5pt, only marks, mark size=3.0pt, mark=x, mark options={solid, mycolor5}]
  table[row sep=crcr]{%
0.487	0.930\\
};

\addplot [color=mycolor2, line width=1.5pt, only marks, mark size=3.0pt, mark=x, mark options={solid, mycolor2}]
  table[row sep=crcr]{%
0.952	0.966\\
};

\addplot [color=mycolor3, line width=1.5pt, only marks, mark size=3.0pt, mark=x, mark options={solid, mycolor3}]
  table[row sep=crcr]{%
1.703	0.982\\
};

\addplot [color=mycolor4, line width=1.5pt, only marks, mark size=3.0pt, mark=x, mark options={solid, mycolor4}]
  table[row sep=crcr]{%
3.110	0.989\\
};

\addplot [color=mycolor5, line width=1.5pt, only marks, mark size=3.0pt, mark=square, mark options={solid, mycolor5}]
  table[row sep=crcr]{%
0.487 0.896\\
};

\addplot [color=mycolor2, line width=1.5pt, only marks, mark size=3.0pt, mark=square, mark options={solid, mycolor2}]
  table[row sep=crcr]{%
0.960 0.946\\
};

\addplot [color=mycolor3, line width=1.5pt, only marks, mark size=3.0pt, mark=square, mark options={solid, mycolor3}]
  table[row sep=crcr]{%
1.744 0.977\\
};

\addplot [color=mycolor4, line width=1.5pt, only marks, mark size=3.0pt, mark=square, mark options={solid, mycolor4}]
  table[row sep=crcr]{%
2.894 0.988\\
};

\addplot [color=mycolor5, line width=1.5pt, only marks, mark size=3.0pt, mark=+, mark options={solid, mycolor5}]
  table[row sep=crcr]{%
0.477 0.934\\
};

\addplot [color=mycolor2, line width=1.5pt, only marks, mark size=3.0pt, mark=+, mark options={solid, mycolor2}]
  table[row sep=crcr]{%
0.914 0.973\\
};

\addplot [color=mycolor3, line width=1.5pt, only marks, mark size=3.0pt, mark=+, mark options={solid, mycolor3}]
  table[row sep=crcr]{%
1.703 0.989\\};

\addplot [color=mycolor4, line width=1.5pt, only marks, mark size=3.0pt, mark=+, mark options={solid, mycolor4}]
  table[row sep=crcr]{%
2.969 0.993\\
};

\addplot [color=mycolor6, dashed, forget plot]
  table[row sep=crcr]{%
0.25	0.6\\
0.25	1\\
};

\addplot [color=mycolor5, dashed, forget plot]
  table[row sep=crcr]{%
0.5	0.6\\
0.5	1\\
};
\addplot [color=mycolor2, dashed, forget plot]
  table[row sep=crcr]{%
1	0.6\\
1	1\\
};
\addplot [color=mycolor3, dashed, forget plot]
  table[row sep=crcr]{%
2	0.6\\
2	1\\
};
\addplot [color=mycolor4, dashed, forget plot]
  table[row sep=crcr]{%
4	0.6\\
4	1\\
};
\end{axis}

\begin{axis}[%
width=1.2in,
height=1in,
at={(1.95in,0.244in)},
scale only axis,
xmode=log,
xmin=0.15,
xmax=170,
xminorticks=true,
ymin=0.78,
ymax=1,
xlabel={\scriptsize{Lipschitz constant}},
xlabel near ticks,
ylabel near ticks,
axis background/.style={fill=white},
legend style={at={(0.96,0.26)}, anchor=east, legend cell align=left, align=left, draw=white!15!black},
every tick label/.append style={font=\footnotesize}
]
\addplot [color=mycolor1, line width=1.5pt, only marks, mark size=3.0pt, mark=o, mark options={solid, mycolor1}]
  table[row sep=crcr]{%
148.029	0.993\\
};
\addlegendentry{\tiny{Vanilla}}

\addplot [color=mycolor5, line width=1.5pt, only marks, mark size=3.0pt, mark=x, mark options={solid, mycolor6}]
  table[row sep=crcr]{%
  0.203	0.854\\
};
\addlegendentry{\tiny{LipKernel}}

\addplot [color=mycolor6, line width=1.5pt, only marks, mark size=3.0pt, mark=square, mark options={solid, mycolor6}]
  table[row sep=crcr]{%
0.227 0.794\\
};
\addlegendentry{\tiny{AOL}}

\addplot [color=mycolor5, line width=1.5pt, only marks, mark size=3.0pt, mark=x, mark options={solid, mycolor5}]
  table[row sep=crcr]{%
0.380	0.886\\
};

\addplot [color=mycolor2, line width=1.5pt, only marks, mark size=3.0pt, mark=x, mark options={solid, mycolor2}]
  table[row sep=crcr]{%
0.759	0.917\\
};

\addplot [color=mycolor3, line width=1.5pt, only marks, mark size=3.0pt, mark=x, mark options={solid, mycolor3}]
  table[row sep=crcr]{%
1.312	0.949\\
};

\addplot [color=mycolor4, line width=1.5pt, only marks, mark size=3.0pt, mark=x, mark options={solid, mycolor4}]
  table[row sep=crcr]{%
2.455	0.971\\
};

\addplot [color=mycolor7, line width=1.5pt, only marks, mark size=3.0pt, mark=x, mark options={solid, mycolor7}]
  table[row sep=crcr]{%
5.401380341	0.984300017\\
};

\addplot [color=mycolor6, line width=1.5pt, only marks, mark size=3.0pt, mark=x, mark options={solid, mycolor6}]
  table[row sep=crcr]{%
8.297845529	0.986800015\\
};

\addplot [color=mycolor5, line width=1.5pt, only marks, mark size=3.0pt, mark=x, mark options={solid, mycolor5}]
  table[row sep=crcr]{%
12.56657442	0.986299992\\
};

\addplot [color=mycolor5, line width=1.5pt, only marks, mark size=3.0pt, mark=square, mark options={solid, mycolor5}]
  table[row sep=crcr]{%
0.464 0.825\\
};

\addplot [color=mycolor2, line width=1.5pt, only marks, mark size=3.0pt, mark=square, mark options={solid, mycolor2}]
  table[row sep=crcr]{%
0.926 0.887\\
};

\addplot [color=mycolor3, line width=1.5pt, only marks, mark size=3.0pt, mark=square, mark options={solid, mycolor3}]
  table[row sep=crcr]{%
1.718 0.930\\
};

\addplot [color=mycolor4, line width=1.5pt, only marks, mark size=3.0pt, mark=square, mark options={solid, mycolor4}]
  table[row sep=crcr]{%
2.939 0.959\\
};

\addplot [color=mycolor7, line width=1.5pt, only marks, mark size=3.0pt, mark=square, mark options={solid, mycolor7}]
  table[row sep=crcr]{%
3.645500618	0.971400023\\
};

\addplot [color=mycolor6, line width=1.5pt, only marks, mark size=3.0pt, mark=square, mark options={solid, mycolor6}]
  table[row sep=crcr]{%
4.505931759	0.98360002\\
};

\addplot [color=mycolor5, line width=1.5pt, only marks, mark size=3.0pt, mark=square, mark options={solid, mycolor5}]
  table[row sep=crcr]{%
5.769053197	0.988900006\\
};

\addplot [color=mycolor6, dashed, forget plot]
  table[row sep=crcr]{%
0.25 	0.6\\
0.25	    1\\
};

\addplot [color=mycolor5, dashed, forget plot]
  table[row sep=crcr]{%
0.5	0.6\\
0.5	1\\
};
\addplot [color=mycolor2, dashed, forget plot]
  table[row sep=crcr]{%
1	0.6\\
1	1\\
};
\addplot [color=mycolor3, dashed, forget plot]
  table[row sep=crcr]{%
2	0.6\\
2	1\\
};
\addplot [color=mycolor4, dashed, forget plot]
  table[row sep=crcr]{%
4	0.6\\
4	1\\
};
\addplot [color=mycolor7, dashed, forget plot]
  table[row sep=crcr]{%
8	0.6\\
8	1\\
};
\addplot [color=mycolor6, dashed, forget plot]
  table[row sep=crcr]{%
16	0.6\\
16	1\\
};
\addplot [color=mycolor5, dashed, forget plot]
  table[row sep=crcr]{%
32	0.6\\
32	1\\
};
\end{axis}\end{tikzpicture}%